\documentclass[11pt]{article}

\usepackage[utf8]{inputenc}
\pdfoutput=1

\usepackage{booktabs} 
\usepackage[authoryear,sort]{natbib}
\usepackage{xfrac}
\usepackage{algorithm}
\usepackage[noend]{algorithmic}
\usepackage{amsmath,amsthm,amsfonts}
\usepackage{xspace}
\usepackage{color}
\usepackage{mathrsfs}
\usepackage{verbatim}
\usepackage{enumitem}
\usepackage{bm}
\usepackage{hyperref}
\usepackage{xr}
\usepackage{tikz}
\usetikzlibrary{bayesnet}
\usepackage{subcaption}
\usepackage{float}
\usepackage{listings}
\usepackage{array}
\usepackage{makecell}
\usepackage[margin=1in]{geometry}
\usepackage{mathtools}

\usepackage{cleveref}

\usepackage{multicol}
\usepackage{esvect}

\definecolor{DarkBlue}{rgb}{0.1,0.1,0.5}
\definecolor{DarkGreen}{rgb}{0.1,0.5,0.1}
\hypersetup{
	colorlinks=true,       
	linkcolor=DarkBlue,          
	citecolor=DarkBlue,        
	filecolor=DarkBlue,      
	urlcolor=DarkGreen,          
	pdftitle={},
	pdfauthor={},
}

\usepackage{csquotes}
\SetBlockThreshold{2}

\usepackage{mdframed}

\usepackage{tikz}

\providecommand{\keywords}[1]
{
  \small	
  \textbf{\textit{Keywords---}} {\textit{#1}}
}

\usetikzlibrary{shapes,decorations,arrows,calc,arrows.meta,fit,positioning}
\tikzset{
    -Latex,auto,node distance =1 cm and 1 cm,semithick,
    state/.style ={ellipse, draw, minimum width = 1 cm,minimum height = 0.9 cm,inner sep=0.02cm},
    point/.style = {circle, draw, inner sep=0.04cm,fill,node contents={}},
    bidirected/.style={Latex-Latex,dashed},
    el/.style = {inner sep=2pt, align=left, sloped}
}

\usepackage{titlesec}
\usepackage{amssymb}

\titleformat*{\paragraph}{\bfseries}

\newcommand{\bigCI}{\mathrel{\text{\scalebox{1.07}{$\perp\mkern-10mu\perp$}}}}

\newcommand{\cX}{{\mathcal X}}
\newcommand{\cY}{{\mathcal Y}}
\newcommand{\cM}{{\mathcal M}}
\newcommand{\cD}{{\mathcal D}}

\newcommand{\cH}{\mathcal{H}}
\newcommand{\cG}{\mathcal{G}}
\newcommand{\cN}{\mathcal{N}}

\newcommand{\cL}{\mathcal{L}}

\newcommand{\argmin}{\mathop{\rm argmin}}

\newcommand{\R}{\mathbb{R}}
\newcommand{\E}{\mathbb{E}}

\newtheorem{example}{Example}[section]

\newtheorem{theorem}{Theorem}
\newtheorem{definition}{Definition}
\newtheorem{proposition}[theorem]{Proposition}

\newtheorem{remark}{Remark}[section]
\newtheorem{lemma}[theorem]{Lemma}

\newtheorem{assumption}{Assumption}

\usepackage{times}

\title{Anticipating Performativity by Predicting from Predictions}

\usepackage{authblk}
\stepcounter{footnote}
\addtocounter{footnote}{-1}
\author[1]{Celestine Mendler-Dünner\footnote{Correspondence to:  \color{DarkBlue}{cmendler@tuebingen.mpg.de}}}
\author[2]{Frances Ding}
\author[3]{Yixin Wang}
\affil[1]{Max-Planck Institute for Intelligent Systems, Tübingen}
\affil[2]{University of California, Berkeley}
\affil[3]{University of Michigan}
\date{}                     
\begin{document}

\maketitle

\vspace{-0.5cm}

\begin{abstract}

Predictions about people, such as their expected educational achievement or their credit risk, can be performative and shape the outcome that they aim to predict. Understanding the causal effect of these predictions on the eventual outcomes is crucial for foreseeing the implications of future predictive models and selecting which models to deploy.
However, this causal estimation task poses unique challenges: model predictions are usually deterministic functions of input features and highly correlated with outcomes. This can make the causal effects of predictions on outcomes impossible to disentangle from the direct effect of the covariates. 
We study this problem through the lens of causal identifiability, and despite the hardness of this problem in full generality, we highlight three natural scenarios where the causal relationship between covariates, predictions and outcomes can be identified from observational data: randomization in predictions, overparameterization of the predictive model deployed during data collection, and discrete prediction outputs. Empirically we show that given our identifiability conditions hold, standard variants of supervised learning that predict from predictions by treating the prediction as an input feature can indeed find transferable functional relationships that allow for conclusions about newly deployed predictive models. These positive results fundamentally rely on \emph{model predictions being recorded during data collection}, bringing forward the importance of rethinking standard data collection practices to enable progress towards a better understanding of social outcomes and performative feedback loops. 
\end{abstract}


\section{Introduction}
Predictions can impact sentiments, alter expectations, inform actions, and thus change the course of events. Through their influence on people, predictions have the potential to change the regularities in the population they seek to describe and understand. This insight underlies the theories of {performativity}~\citep{mackenzie2008engine} and reflexivity~\citep{soros15alchemy} that play an important role in modern economics and finance. 

Recently, \citet{perdomo20pp} pointed out that the social theory of performativity has important implications for machine learning theory and practice. 
Prevailing approaches to supervised learning assume that features $X$ and labels $Y$ are sampled jointly from a fixed underlying data distribution that is unaffected by attempts to predict $Y$ from $X$. Performativity questions this assumption and suggests that the deployment of a predictive model can disrupt the relationship between $X$ and $Y$. Hence, changes to the predictive model can induce shifts in the data distribution. For example, consider a lender with a predictive model for risk of default -- performativity could arise if individuals who are predicted as likely to default are given higher interest loans, which make default even more likely \citep{manso13credit}, akin to a self-fulfilling prophecy. In turn, a different predictive model that predicts smaller risk and suggests offering more low-interest loans could cause some individuals who previously looked risky to be able to pay the loans back, which would appear as a shift in the relationship between features $X$ and loan repayment outcomes $Y$. This performative nature of predictions poses an important challenge to using historical data to predict the outcomes that will arise under the deployment of future models.

\subsection{Our work}
In this work, we aim to understand under what conditions observational data is sufficient to identify the performative effects of predictions. Only when causal identifiability is established can we rely on data-driven strategies to anticipate performativity and reason about the downstream consequences of deploying new models. 
Towards this goal, we focus on a subclass of performative prediction problems where performative effects are mediated by predictions,  surface as a shift in the outcome variable, and the distribution over covariates $X$ is unaffected by prediction.
Our goal is to identify the expected counterfactual outcome
\[\mathcal M_Y(x,\hat y)\triangleq\E[Y|X=x, \text{do}(\hat Y=\hat y)].\]
Understanding the causal mechanism $\mathcal M_Y$ is crucial for model evaluation, as well as model optimization. In particular, it allows for offline evaluation of the potential outcome $Y$ of an individual $x$ subject to any unseen predictive model $f_\text{new}$ before actually deploying it, by simply plugging in the prediction $\hat y=f_\text{new}(x)$. 

\paragraph{The need for observing predictions.} We start by illustrating the hardness of performativity-agnostic learning by relating  performative prediction to a concept shift problem; with every model deployment a potentially different distribution over covariates and labels is induced. Using the structural properties of performative distirbution shifts, we establish a lower bound on the extrapolation error of predicting $Y$ from $X$ under the deployment of a model $f_\text{new}$ that is different from the model $f_\text{train}$ deployed during data collection. The extrapolation error grows with the distance between the predictions of the two models and the strength of performativity. This lower bound on the extrapolation error demonstrates the necessity to take performativity into account for reliably predicting the outcome $Y$. 

\paragraph{Predicting from predictions.} We then explore the feasibility of identifying performative effects when the training data recorded the predictions $\hat Y$ and  training data samples $(X,Y,\hat Y)$ are available. As a concrete identification strategy for learning $\cM_Y(x,\hat y)$ we focus on building a meta machine learning model that predicts the outcome $Y$ for an individual with features $X$, subjected to a prediction $\hat Y$. We term this data-driven strategy \emph{predicting from predictions} because it treats the predictions as an input to the meta machine learning  model. 
The meta model seeks to answer ``what would the outcome be if we were to deploy a different prediction model?'' Crucially, this ``what if'' question is causal in nature; it aims to understand the potential outcome of the intervention where we deploy a predictive model different from the one in the training data; this goal is different from merely estimating the outcome variable in previously seen data. 
Whether such a transferable model is learnable depends on whether the training data provides causal \emph{identifiability}~\citep{pearl2009}. 
Only after causal identifiability is established can we rely on observational data to select and design optimal downstream predictive models under performativity.

\paragraph{Establishing identifiability.} For our main technical results, we first show that, in general, observing $\hat Y$ is \emph{not} sufficient for identifying the causal effect of predictions. In particular, when the training data was collected under the deployment of a deterministic prediction function $f_\text{train}$, the mechanism $\mathcal M_Y$ can not be uniquely identified. The reason is that a lack of coverage in the training data---the covariates $X$ and the prediction $\hat Y$ are deterministically bound---prohibits causal identification. 
Next, we establish several conditions under which observing $\hat Y$ is sufficient for identifying $\cM_Y$.
The first condition exploits randomness in the prediction. This randomness could be purposely built into the prediction for individual fairness, differential privacy, or other considerations. 
The second condition exploits the property that predictive models are often over-parameterized, which leads to incongruence in functional complexity between different causal paths; such incongruence enables the effects of predictions to be separated from other variables' effects. The third condition takes advantage of discreteness in predictions such that performative effects can be disentangled from the continuous relationship between covariates and outcomes. Taken together, the conditions we identified reveal that natural inaccuracies and particularities of prediction problems can provide causal identifiability of performative effects. This implies that there is hope that we can recover the causal effect of predictions from observational data.
In particular, we show that, under these conditions, standard supervised learning can be used to find transferable functional relationships by treating predictions as model inputs, even in finite samples.


\paragraph{Discussion and future work.} We conclude with a discussion of limitations and extensions of our work by explaining potential violations of the modeling assumptions underlying our causal analysis. This opens up interesting directions for future work, including the study of spill-over effects in prediction, performativity in non-causal prediction, and causal identifiability of performative effects under performative covariate shift.  

\subsection{Broader context and related work}

The work by \citet{perdomo20pp}, initiated the discourse of performativity in the context of supervised learning by pointing out that the deployment of a predictive model can impact the data distribution we train our models on. Existing scholarship on performative prediction~\citep[c.f.,][]{perdomo20pp,mendler20stochastic,drusvyatskiy20stochastic,miller21echo, izzo21perf,jagadeesan22bandit, wood22, narang22, piliouras22, kulynych2022causal} has predominantly focused on achieving a particular solution concept with a prediction function that maps $X$ to $Y$ in the presence of unknown performative effects.
Complementary to these works we are interested in understanding the underlying causal mechanism of the performative distribution shift, so we can account for these shifts when designing new models. Our work is motivated by the seemingly natural approach of lifting the supervised-learning problem and incorporating the prediction as an input feature when building a meta machine learning model for explaining the outcome $Y.$ By establishing a connection to causal identifiability, our goal is to understand when such a data-driven strategy can be helpful for finding transferrable functional relationships between $X$, $\hat Y$ and $Y$ that enable us to anticipating the down-stream effects of prediction.

This work focuses on the setting where performativity only surfaces in the label, while the marginal distribution $P(X)$ over covariates is assumed to be fixed. This represents a  subclass of performative (aka. model-induced or decision-dependent) distribution shift problems~\citep{perdomo20pp,liu2021induced, drusvyatskiy20stochastic}. In particular, our assumptions are complementary to the strategic classification framework~\citep{brueckner12strat,Hardt:2016:SC} that focuses on a setting where performative effects concern $P(X)$, while $P(Y|X)$ is assumed to remain stable. Consequently, causal questions in strategic classification~\citep[e.g.,][]{harris21,bechavod20help,shavit20a}
are concerned with identifying stable causal relationships between $X$ and $Y$. Since we assume $P(Y|X)$ can change as a result of model deployment (i.e. the true underlying 'concept' determining outcomes can change), conceptually different questions emerge in our work. Similar in spirit to strategic classification, the work on algorithmic recourse and counterfactual explanations~\citep{thibault18inv,karimi21recourse, stratis20} focuses on the causal link between features and predictions, whereas we focus on the down-stream effects of predictions. 

There are interesting parallels between our work and related work on the offline evaluation of online policies~\citep[e.g.,][]{li11,swaminathan15,li15,schnabel16}. In particular, \cite{swaminathan15} explicitly emphasize the importance of logging propensities of the deployed policy during data collection to be able to  mitigate selection bias. In our work the deployed model can induce a concept shift. Thus, we find that additional information about the predictions of the deployed model needs to be recorded to be able to foresee the impact of a new predictive model on the conditional distribution $P(Y|X)$, beyond enabling propensity weighting~\citep{rosenbaum83}.  A notable work by \cite{wager14} investigates how predictions at one time step impact predictions in future time steps. Our problem formulation is different in that we aim to understand the causal effect of $\hat Y$ on $Y$ which can not be inferred solely by studying sequences of predictions. 
Furthermore, complementary to these existing works we show that randomness in the predictive model is not the only way causal effects of predictions can be identified.

For our theoretical results, we build on classical tools from causal inference~\citep{pearl09causality,rubin80,tchetgen2012causal}, and establish a connection to more recent identification techniques by \cite{eckels20noise}. In particular, we distill unique properties of the performative prediction problem to design assumptions for the identifiability of the causal effect of predictions. 

\section{The causal force of prediction}

Predictions can be performative and impact the population of individuals they aim to predict. %
Through the lens of causal inference~\citep{pearl09causality}, the deployment of a predictive model in performative prediction represents an intervention. Namely, an intervention on a causal diagram that describes the underlying data generation process of the population. 
In the following we will build on this causal perspective to study an instance of the performative prediction problem and elucidate the hardness of performativity-agnostic learning.

\subsection{Prediction as a partial mediator}
Consider a machine learning application relying on a predictive model $f$ that maps features $X$ to a predicted label $\hat Y$. 
We assume the predictive model $f$ is performative in that the prediction $\hat Y=f(X)$ has a direct causal effect on the outcome variable $Y$ of the individual it concerns. Thereby the prediction impacts how the outcome variable $Y$ is generated from the features $X$. The causal diagram illustrating this setting is below: 
\vspace{-0.2cm}
\begin{figure}[H]
\centering
\begin{minipage}{0.4\textwidth}
\begin{figure}[H]
\centering
\begin{tikzpicture}
    \node[state] (x) at (0,0) {$X$};
    \node[state] (y) [right =of x] {$Y$};
    \node[state] (S) [above right =of x, xshift=-0.7cm,yshift=-0.3cm] {$\hat Y$};
    \node[state] (f) [above left =of x, xshift= 0.7cm,yshift=-0.3cm] {$f$}; 
    \path (x) edge (y);
    \path (x) edge (S);
    \path (f) edge (S);
    \path (S) edge (y);
\end{tikzpicture}
\end{figure}
\end{minipage} 
\begin{minipage}{0.5\textwidth}
\begin{align}
    X &= \xi_X\quad &\;\xi_X\sim\cD_X \label{eq:SEMa1}\\
    \hat Y &=f(X, \xi_f) &\xi_f\sim\cD_f\; \\
    Y&= g(X,\hat Y)+\xi_Y &\xi_Y\sim\cD_Y\label{eq:SEMa-1}
\end{align}
\end{minipage}
\caption{\small{Performative effects in the outcome mediated by the prediction for a given $f$}} 
\label{fig:haty}
\end{figure}

The features $X\in\cX\subseteq \R^d$ are drawn i.i.d.~from a fixed underlying continuous distribution over covariates $\cD_X$ with support $\cX$.  
The outcome $Y\in\cY\subseteq \R$ is a function of $X$, partially mediated by the prediction $\hat Y\in\cY$. 
The prediction $\hat Y$ is determined by the deployed predictive model $f:\cX\rightarrow \cY$.
For a given prediction function $f$, every individual is assumed to be sampled i.i.d. from the data generation process described by the causal graph in Figure~\ref{fig:haty}. We assume the exogenous noise $\xi_Y$ is zero mean, and $\xi_f$ allows the prediction function to be randomized. 
This setup differs from the traditional supervised learning setting by including the arrow between $\hat{Y}$ and $Y$ in the causal graph.

Note that our model is not meant to describe performativity in its full generality (which includes other ways the predictive model $f$ may affect $P(X,Y)$). Rather, it describes an important and practically relevant class of performative feedback problems that are characterized by two properties: 1) performativity surfaces only in the label $Y$, and 2) performative effects are mediated by the prediction, such that $Y\bigCI f\;|\;\hat Y$, rather than dependent on the specifics of $f$.

\paragraph{Application examples.} Causal effects of predictions on outcomes have been documented in various contexts: A bank's prediction about the client (e.g., his or her creditworthiness in applying for a loan) determines the interest rate assigned to them, which in turn changes a client's financial situation~\citep{manso13credit}. Mathematical models that predict stock prices inform the actions of traders and thus heavily shape financial markets and economic realities~\citep{mackenzie2008engine}. Zillow's housing price predictions directly impact sales prices~\citep{Malik2020DoesML}. Predictions about the severity of an illness play an important role in treatment decisions and hence the very chance of survival of the patient~\citep{levin18machine}. Another prominent example from psychology is the Pygmalion effect~\citep{rosenthal68}. It refers to the phenomenon that high expectations lead to improved performance, which is widely documented in the context of education~\citep{xander09pygmalion}, sports~\citep{solomon96}, and organizations~\citep{eden92}. 
Examples of such performativity abound, and we hope to have convinced the reader that the performative effects in the outcome that we study in this work are important for algorithmic prediction.

\subsection{Implications for performativity-agnostic learning}
\label{sec:concept}

Begin with considering the classical supervised learning task where data about $X,Y$ is available and $\hat Y$ is unobserved. The goal is to learn a model $h:\cX\rightarrow\cY$  for predicting the label $Y$ from the features $X$. To understand the inherent challenge of classical prediction under performativity, we investigate the relationship between $X$ and $Y$ more closely. Specifically, the structural causal model (Figure~\ref{fig:haty}) that describes the data generation process implies that
\begin{equation}
     \label{eq:marginal}
     P(Y|X) = \int P(Y|\hat{Y}, X) P(\hat{Y}|X) \mathrm{d} \hat{Y}.
\end{equation}
This expression makes explicit how the relationship between $X$ and $Y$ that we aim to learn depends on the predictive model governing $P(\hat{Y}|X)$. As a consequence, when the deployed predictive model at test time differs from the model deployed during training data collection, performative effects surface as concept shift~\citep{gama14concept}.
Such transfer learning problems are known to be intractable without structural knowledge about the distribution shift, implying that we can not expect $h$ to generalize to distributions induced by future model deployments. 
Let us inspect the resulting extrapolation gap in more detail and put existing positive results on performative prediction into perspective.

\paragraph{Extrapolation loss.} We illustrate the effect of performativity on predictive performance using a simple instantiation of the structural causal model from Figure~\ref{fig:haty}. Therefore, assume a linear performative effect of strength $\alpha>0$ and a base function $g_1:\cX\rightarrow \cY$
\begin{equation}
    \label{eq:linfY}
    g(X,\hat Y):= g_1(X)+\alpha \hat Y.
\end{equation}
Now, assume we collect training data under the deployment of a predictive model $f_\theta$ and validate our model under the deployment of $f_\phi$.  Using our running example of a lender predicting the risk of default, the lender may have historical data about individuals who defaulted or not. Given this data the lender aims to learn a model to predict whether similar individuals will default in the future. However, in the time between data collection and model validation, the predictive model for allocating interest rates might have been updated. If not accounted for, the resulting effects of the change in the interest rate on an individual's default risk will be perceived by the lender as extrapolation loss.

To quantify the extrapolation loss, we adopt the notion of a distribution map from~\citet{perdomo20pp} and write $\cD_{XY}(f)$ for the joint distribution over $(X,Y)$ surfacing from the deployment of a model $f$. We assess the quality of our predictive model $h:\cX\rightarrow \cY$ over a distribution $\cD_{XY}(f)$ induced by $f$ via the loss function $\ell: \cY\times \cY\rightarrow \mathbb R$ and write $\mathrm{R}_{f}(h):=\mathrm E_{x,y\sim \cD_{XY}(f)} \ell(h(x),y)$ for the risk of $h$ on the distribution induced by $f$. We use $h_f^*$ for the risk minimizer $h_f^*:=\argmin_{h\in\cH} \mathrm{R}_{f}(h)$, and $\cH$ for the hypothesis class we optimize over. The following result shows that the extrapolation loss of a model optimized over $\cD_{XY}(f_\theta)$ and evaluated on $\cD_{XY}(f_\phi)$ grows with the strength of performativity and the distance between $f_\theta$ and $f_\phi$ as measured in prediction space.  Proposition~\ref{propo:ext} can be viewed as a concrete instantiation of the more general extrapolation bounds for performative prediction discussed in~\citep{liu2021induced} within the feedback model from Figure~\ref{fig:haty}.

\begin{proposition}[Hardness of performativity-agnostic prediction]
\label{propo:ext}
Consider the data generation process in Figure~\ref{fig:haty} with $g$ given in \eqref{eq:linfY} and  $f_\theta, f_\phi$ being deterministic functions.  Take a loss function $\ell: \cY\times\cY\rightarrow \mathbb R$ that is $\gamma$-smooth and $\mu$-strongly convex in its second argument. Let $h_{f_\theta}^*$ be the risk minimizer over the training distribution and assume the problem is realizable, i.e., $h_{f_\theta}^*\in\cH$. 
Then, we can bound the extrapolation loss of $h_{f_\theta}^*$ on the distribution induced by $f_\phi$ as
\begin{align}
\frac \gamma 2   \alpha^2 d^2_{\cD_X}(f_\theta,f_\phi)\geq \Delta R_{f_\theta\rightarrow f_\phi}(h_{f_\theta}^*)\geq 
\frac \mu 2   \alpha^2 d^2_{\cD_X}(f_\theta,f_\phi)
\label{eq:experror}
\end{align}
where $ d^2_{\cD_X}(f_\theta,f_\phi):= \mathrm E_{x\sim \cD_X} (f_\theta(x)-f_\phi(x))^2$ and  $\Delta R_{f_\theta\rightarrow f_\phi}(h):=\mathrm{R}_{f_\phi}(h)-\mathrm{R}_{f_\theta}(h)$.
\end{proposition}
The extrapolation loss $\Delta R_{f_\theta\rightarrow f_\phi}(h_{f_\theta}^*)$ is zero if and only if either the strength of performativity tends to zero ($\alpha\rightarrow 0$), or the predictions of the two predictors $f_\theta$ and $f_\phi$ 
are identical over the support of $\cD_X$. If this is not the case, an extrapolation gap is inevitable. This elucidates the fundamental hardness of performative prediction from feature, label pairs $(X,Y)$ when performative effects disrupt the causal relationship between $X$ and $Y$. 

The special case where $\alpha=0$ aligns with the assumption of classical supervised learning, in which there is no performativity. This may hold in practice if the predictive model is solely used for descriptive purposes, or if the agent making the prediction does not enjoy any economic power~\citep{hardt22power}.  However, the strength of performative effects is not a parameter we can influence as machine learning practitioners and thus we work under the assumption that any prediction can be performative.

The second special case where the extrapolation error $\Delta R_{f_\theta\rightarrow f_\phi}(h_{f_\theta}^*)$ is small is when $d^2_{\cD_X}(f_\theta,f_\phi)\rightarrow0$. Given our causal model, this implies that $\cD_{XY}(f_\theta)$ and $\cD_{XY}(f_\phi)$ are equal in distribution and hence exhibit the same risk minimizer. Such a scenario where $f_\theta$ and $f_\phi$ are similar can happen, for example, if the model $f_\phi$ is obtained by retraining $f_\theta$ on observational data and a fixed point is reached where $f_\theta=h_{f_\theta}^*$ (also known as performative stability~\citep{perdomo20pp}). The convergence of different policy optimization strategies to stable points has been studied in prior work~\citep[e.g.,][]{perdomo20pp, mendler20stochastic,drusvyatskiy20stochastic} and enabled optimality results even in the presence of performative concept shifts, relying on the target model $f_\phi$ not being chosen arbitrarily, but based on a pre-specified update strategy.

\section{Identifying the causal effect of prediction}
\label{sec:identify}

Having illustrated the hardness of performativity-agnostic learning, we explore under what conditions incorporating the presence of performative predictions into the learning task enables us to recover the transferrable causal mechanism $\cM_Y$ for explaining $Y$.
Towards this goal, a necessary first step is to assume that the mediator $\hat Y$ in Figure~\ref{fig:haty} is observed---the prediction takes on the role of the treatment in our causal analysis and we can not possibly hope to estimate the treatment effect of a treatment that is unobserved.

\subsection{Problem setup}
\label{sec:setup}
Assume we are given access to data points $(x,\hat y,y)$ generated i.i.d.~from the structural causal model in Figure~\ref{fig:haty} under the deployment of a prediction function $f_\theta$. 
From this observational data, we wish to estimate the expected potential outcome of an individual under the deployment of an unseen (but known) predictive model $f_\phi$. We note that given our causal graph, the implication of intervening on the function $f$ can equivalently be explained by an intervention on the prediction $\hat Y$. Thus, we are interested in identifying the causal mechanism:
\begin{equation}
\label{eq:h*}
    \cM_Y(x,\hat y):=\E[Y|X=x,\mathrm{do}(\hat Y = \hat y)].
\end{equation}
Unlike $P(Y|X)$, the mechanism $\cM_Y(x,y)$ is invariant to the changes in the predictive model governing $P(\hat{Y}|X)$.
Thus, being able to identify $\cM_Y$ will allow us to make inferences about the potential outcome surfacing from planned model updates beyond explaining historical data. In particular, we can evaluate $\cM_Y$ to infer the potential outcome $y$ for any $x$ at $\hat{y} = f_\phi(x)$ for $f_\phi$ being the model of interest.

For simplicity of notation, we will write $\cD(f_\theta)$ to denote the joint distribution over $(X,\hat Y,Y)$ of the observed data collected under the deployment of the predictive model $f_\theta$.
We say $\cM_Y$ can be identified, if it can uniquely be expressed as a function of observed data. More formally: 

\begin{definition}[identifiability]
Given a predictive model $f$, the causal graph in Figure~\ref{fig:haty}, and a set of assumptions $A$. We say the causal mechanism $\cM_Y$ is identifiable from $\cD(f)$, if for any function $h$ that complies with assumptions $A$ and $h(x,\hat y)=\cM_Y(x,\hat y)$ for pairs $(x,\hat y)\in\mathrm{supp}(\cD_{X,Y}(f))$  it must also  hold that $h(x,\hat y)=\cM_Y(x,\hat y)$ for all pairs $(x,\hat y)\in\cX\times\cY$.
\end{definition}


Without causal identifiability, there might be other models $h\neq \cM_Y$ that explain the training distribution equally well but do not transfer to the distribution induced by the deployment of a new model. Causal identifiability is crucial for extrapolation and for using $\cM_Y$ to draw conclusions about the outcome under unseen models. It quantifies the limits of what we can infer given access to the training data distribution, ignoring finite sample considerations.

\begin{remark}[Alternate objectives]
Instead of the expected potential outcome $\cM_Y(x,\hat y)$  we might be interested in an alternate causal quantity $\E[\kappa(X,Y,\hat Y)|X=x,\mathrm{do}(\hat Y=\hat y)]$ instead. The function $\kappa$ could measures the loss of predictions, individual improvement, or other goals for socially beneficial machine learning that an auditor or a model designer is interested in. The technical criteria for identifiability of the causal effect established in this work would remain the same, as long as $\kappa$ is a continuous function. 
\end{remark}

\paragraph{Identification with supervised learning.} Identifiability guarantees of $\cM_Y$ from samples of $\cD(f_\theta)$ imply that the historical data collected under the deployment of $f_\theta$ contains sufficient information to recover the invariant relationship~\eqref{eq:h*}. As a concrete identification strategy, consider the following standard variant of supervised learning that takes in samples $(x,\hat y,y)$ and builds a meta-model that predicts $Y$ from $X,\hat Y$ by solving the following risk minimization problem
\begin{equation}
    h_\text{SL}:=\argmin_{h\in \mathcal H} \;\mathrm E_{(x,\hat y, y)\sim\cD(f_\theta)} \big[
    \left(h(x,\hat y)-y \right)^2\big].\label{eq:obj}
\end{equation}
where $\cH$ denotes the hypothesis class. We consider the squared loss for risk minimization because it pairs well with the exogeneous noise $\xi_Y$ in \eqref{eq:SEMa-1} being additive and zero mean. The optimization strategy~\eqref{eq:obj} is an instance of what we term \emph{predicting from predictions}. Lemma~\ref{lem:equal} provides a sufficient condition for the supervised learning solution $h_\text{SL}$ to recover the invariant causal quantity $\cM_Y$.

\begin{lemma}[Identification strategy]
\label{lem:equal}
Consider the data generation process in Figure~\ref{fig:haty} and a set of assumptions $A$. Given a hypothesis class $\cH$ such that every $h\in\cH$ complies with $A$ and the problem is realizable, i.e., $\cM_Y\in\cH$. Then, if $\cM_Y$ is causally identifiable from $\cD(f_\theta)$ given $A$, the risk minimizer $h_\text{SL}$ in \eqref{eq:obj} will coincide with $\cM_Y$.
\end{lemma}

\subsection{Challenges for identifiability}
The main challenge for identification of $\cM_Y$ from data is that in general, the prediction rule $f_\theta$ which produces $\hat Y$ is a deterministic function of the covariates $X$. This means that, for any realization of $X$, we only get access to one particular $\hat Y=f_\theta(X)$ in the training distribution, which makes it challenging to disentangle the direct effect of $X$ on $Y$ from the indirect effect mediated by $\hat Y$. To illustrate this challenge, consider the function $h(x,\hat y ):=\cM_Y(x,f_\theta(x))$ that ignores the input parameter $\hat y$ and only relies on $x$ for explaining the outcome. This function explains $y$ in the training data equally well and can not be differentiated from $\cM_Y$ based on data collected under the deployment of a deterministic prediction rule $f_\theta$. The problem is akin to  fitting a linear regression model to two perfectly correlated covariates.   More broadly, this ambiguity is due to what is known as a \emph{lack of overlap} (or lack of positivity) in the literature of causal inference~\citep{pearl1995causal,imbens2015causal}. It persists as long as $P(X|\hat Y=\hat y)$ and $P(X|\hat Y=\hat y')$ in the observed distribution do not have common support for pairs of predictions $\hat y, \hat y'$ we are potentially interested in.  In the covariate shift literature, the lack of overlap surfaces when the covariate distribution violates the common support assumption and the propensity scores are not well-defined ~(see e.g., \citet{pan10transfer}). This problem renders causal identification and thus data-driven learning of performative effects from deterministic predictions fundamentally challenging.

\begin{proposition}[Hardness of identifiability from deterministic predictions]
\label{propo:deterministic}
Consider the structural causal model in Figure~\ref{fig:haty}. Assume  $Y$ non-trivially depends on $\hat Y$, and the set $\cY$ is not a singleton. Then, given a deterministic prediction function $f$, the causal quantity $\cM_Y$ is not identifiable from $\cD(f)$. 
\end{proposition}

The identifiability issue persists as long as the two variables $X$, $\hat Y$ are deterministically bound and there is no incongruence or hidden structure that can be exploited to disentangle the direct effect of $X$ on $Y$ from the indirect effect mediated by $\hat Y$. In the following, we focus on particularities of prediction problems and show how they allows us to identify $\cM_Y$.

\subsection{Identifiability from randomization}
\label{sec:random}
We start with the most natural setting that provides identifiability guarantees: randomness in the prediction function $f_\theta$. Using standard arguments about overlap we can identify $\cM_Y(x,\hat y)$ for any pair $x,\hat y$ with positive probability in the data distribution $\cD(f_\theta)$ from which the training data is sampled. To relate this to our goal of identifying the outcome under the deployment of an unseen model $f_\phi$ we introduce the following definition:

\begin{definition}[output overlap]
\label{ass:yoverlap}
Given two predictive models $f_\theta, f_\phi$, the model $f_\phi$ is said to satisfy output overlap with $f_\theta$, if for all $x\in\cX$ and any subset $\mathcal Y'\subseteq\mathcal Y$ with positive measure, it holds that
\begin{equation}
\label{eq:ratio}
    \frac{\mathrm P[f_{\phi}(x)\in \mathcal Y']}{\mathrm P[f_\theta(x)\in \mathcal Y']}>0.
\end{equation}
\end{definition}

In particular, output overlap requires the support of the new model's predictions $f_\phi(x)$ to be contained in the support of $f_\theta(x)$ for every potential $x\in\mathcal X$. The following proposition takes advantage of the fact that the joint distribution over $(X,Y)$ is fully determined by the deployed model's predictions to relate output overlap to identification:

\begin{proposition}[Identifiability from output overlap]
\label{prop:dp}
Given the causal graph in Figure~\ref{fig:haty}, the causal quantity $\cM_Y(x,\hat y)$ is identifiable from $\cD(f_\theta)$ for any pair $x,\hat y$ with  $\hat y =f_\phi(x)$, as long as $f_\phi$ is a prediction function that satisfies output overlap with $f_\theta$. 
\end{proposition}

Proposition~\ref{prop:dp} allows us to pinpoint the models $f_\phi$ to which we can extrapolate to from data collected under $f_\theta$. 
Furthermore, it makes explicit that data collected under the deployment of a fully randomized prediction function $f_\theta$ that attains each value in $\cY$ with non-zero probability for any $x\in\cX$ is ideal for learning and allows for global identification of $\cM_Y$.
Akin to domain randomization for zero-shot transfer learning~\citep{tobin17}, randomization in the prediction gives rise to a dataset that allows for more robust conclusions about the distribution induced by unknown future deployable models $f_\phi$. 
In the context of performative prediction, one natural setting that leads to randomization is the differentially private release of predictions through an additive noise mechanism applied to the output of the prediction function~\citep{dwork06noise}. Here, instead of $\hat Y_\text{orig}=f_\theta(X)$, a noisy version $\hat Y=\hat Y_\text{orig}+\eta$ with $\eta\sim\text{Lap}(0,b)$ for an appropriately chosen $b>0$ is released. Since the Laplace noise has full support, output overlap and identification is guaranteed by Proposition~\ref{prop:dp} for any $f_\phi$. 
Similarly, noise with bounded support would allow for `local' identifiability and extrapolation to models $f_\phi$ that are sufficiently similar in prediction space.

While standard in the literature and natural in certain settings, a caveat of identification from randomization is that there are several reasons a decision-maker may choose not to deploy a randomized prediction function in performative environments, including negative externalities and concerns about user welfare~\citep{KramerGuHa14}, but also business interests to preserve consumer value of the prediction-based service offered.
In the context of our credit scoring example, random predictions would imply that interest rates are randomly assigned to applicants in order to learn how the rates impact their probability of paying back. We can not presently observe this scenario, given regulatory requirements for lending institutions.  Before we turn to scenarios where we can achieve identifiability without randomization of $f_\theta$, we discuss two additional, natural sources of randomness that, combined with side-information, could provide identification.

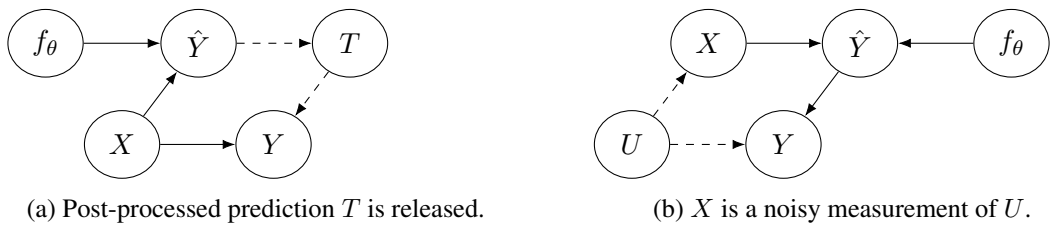
\begin{figure}[t]
    \centering

    \begin{subfigure}[b]{0.4\textwidth}
\begin{tikzpicture}
    \node[state] (x) at (0,0) {$X$};
    \node[state] (y) [right =of x] {$Y$};
    \node[state] (S) [above right =of x, xshift=-0.7cm,yshift=-0.3cm] {$\hat Y$};
    \node[state] (dy) [right =of S] {$T$};
    \node[state] (f) [above left =of x, xshift= 0.7cm,yshift=-0.3cm] {$f_\theta$};
    \path (x) edge (y);
    \path (x) edge (S);
    \path (f) edge (S);
    \path [dashed] (S) edge (dy);
    \path [dashed] (dy) edge (y);
\end{tikzpicture}
\caption{Post-processed prediction $T$ is released.}
\end{subfigure}
\hspace{1cm}
\begin{subfigure}[b]{0.4\textwidth}
\begin{tikzpicture}
    \node[state] (x) at (0,0) {$U$};
    \node[state] (y) [right =of x] {$Y$};
    \node[state] (S) [above right =of x, xshift=-0.7cm,yshift=-0.3cm] {$X$};
    \node[state] (dy) [right =of S] {${\hat{Y}}$};
    \node[state] (f) [right =of dy] {$f_\theta$};
    \path [dashed] (x) edge (y);
    \path [dashed] (x) edge (S);
    \path (f) edge (dy);
    \path (S) edge (dy);
    \path (dy) edge (y);
\end{tikzpicture}
\caption{$X$ is a noisy measurement of $U$.}
\label{fig:U}
\end{subfigure}
\caption{\small{Examples for additional sources of randomness beyond our model.}} %
\label{fig:proxy}
\end{figure}

\paragraph{Alternate sources of randomness in prediction.} If additional side-information, observations, or more fine-grained knowledge about the causal graph structure is available, then identification can also be achieved from other sources of randomness. However, incorporating such side-information requires going beyond standard ERM which is not the main focus of this work. Nevertheless we provide a discussion for completeness. For example, consider the causal graph in Figure~\ref{fig:proxy}(a) where the performative effect of predictions is mediated by a down-stream decision $T\in\{0,1\}$, such that $Y\bigCI\hat Y\,|\,T, X$.  In this case, randomness in the discrete decision function $T$ (instead of the continuous prediction $\hat Y$) is sufficient for identification of the causal graph. 
Randomness in prediction-based decisions can be a deliberate part of an algorithmic system for a number of reasons, including designing individually fair decision rules~\citep{dwork12fairness,berger20random,berger20gender}. 

A second natural source of randomness in performative prediction is noise in the measurement of the covariates $X$, representing the unobserved true underlying attributes $U$. This scenario is illustrated in Figure~\ref{fig:proxy}(b). For example, a student's college performance depends on their underlying scholastic ability, but predictions of performance (and perhaps admissions decisions) are made based on a noisy proxy like SAT score.  In this case, side-information about the structure of the measurement noise enables identification~\citep{eckels20noise} without precise knowledge of $U$. The intuition is that the attributes $U$ that are causal for the outcome $Y$ enter the prediction through the noisy measurements $X$, which adds independent variation to the indirect causal path.  


\subsection{Identifiability through overparameterization}
\label{sec:sep}

The following two sections consider situations where we can achieve identification, without overlap, from data collected under the deployment of a deterministic $f_\theta$. Our first result exploits incongruences in functional complexity arising from machine learning models that are overparameterized, which is common in modern machine learning  applications~\cite[e.g.][]{alex12}. By overparameterization, we refer to the fact that the representational complexity of the model is larger than the underlying concept that needs to be described. We formalize this as follows:

\begin{assumption}[overparameterization]
\label{ass:overparam}
We say a function $f$ is overparameterized with respect to $\mathcal G$ over $\cX$ if there is no function $g'\in\mathcal G$ and $c\in \mathbb{R}$ such that $f(x)=c \cdot g'(x)$ for all $x\in\mathcal X$.
\end{assumption}

For the purpose of this section, assume the structural equation for how $Y$ is generated is separable and has the following form $g(X,\hat Y)= g_1(X)+\beta \hat Y$, where $\hat Y$ is the output of the prediction function $f_\theta$  mapping $X$ to $\hat Y$, and $\beta\geq 0$ is a constant. As we have emphasized earlier, the challenge for identification is that for deterministic $f_\theta$ the prediction can be reconstructed from $X$ without relying on $\hat Y$ and thus the function $h(x,\hat y)=g_1(x)+\beta f_\theta(x)$ can not be differentiated from $\cM_Y$ based on observational  data. For our next identifiability result the key observation is that this ambiguity relies on there being an $h\in\cH$ such that $h(\cdot,\hat y)$ for a fixed $\hat y$ can represent $f_\theta$. In contrast, for prediction functions $f_\theta\notin\cH$, the solution $h_\text{SL}$ (for a well specified $\cH$) will necessarily rely on $\hat Y$ to explain the effect of the prediction. 
To make this intuition more concrete, consider the following example:

\begin{example}
Assume the structural equation for $Y$ in Figure~\ref{fig:haty}  is given as $g(x,\hat y)=\alpha x +\beta \hat y$ for some unknown $\alpha,\beta$. Consider prediction functions $f_\theta$ of the following form $f_\theta(x)=\gamma x^2+\xi x$ for some $\gamma, \xi\geq0$. Consider $\cH$ be the class of linear functions. Then, any consistent estimate $h\in\cH$ takes the form $h(x,\hat y)=\alpha' x +\beta' \hat y$. Furthermore, for $h$ to be consistent with observations we need $ \alpha'+\beta'\xi = \alpha + \beta\xi$ and $\beta'\gamma=\beta\gamma$. This system of equations has a unique solution as long as $\gamma>0$ which corresponds to the case where $f_\theta$ is overparameterized with respect to $\cH$. In contrast, for $\gamma=0$ the function $h(x,\hat y)=(\alpha+\beta\xi) x $ would explain the training data equally well.
\end{example}

The following result generalizes this argument:

\begin{proposition}[Identifiability from overparameterization]
\label{prop:over}
 Consider the structural causal model in Figure~\ref{fig:haty} where $f_\theta$ is a deterministic function. Assume that $g$ can be decomposed as $g(X,\hat Y)=g_1(X)+\alpha \hat Y$ for some $\alpha> 0$ and $g_1\in\cG$, where the function class $\cG$ is closed under addition (i.e. $g_1, g_2\in \cG\Rightarrow a_1\cdot g_1 +a_2 \cdot g_2 \in \cG \quad \forall a_1, a_2\in \mathbb{R}$). Let $\cH$ contain functions that are separable in $X$ and $\hat Y$, linear in $\hat Y$, and $\forall h\in\cH$ it holds that $h(\cdot,\hat y)\in\cG$ for a fixed $\hat y$. Then, if $f_\theta$ is overparameterized with respect to $\mathcal G$ over the support of $\cD_X$, $\cM_Y$ is identifiable from $\cD(f_\theta)$. 
\end{proposition}

The above result can be extended to more general structural causal models of the form $g(X,\hat Y) = g_1(X)+g_2(\hat Y)$. In this case linear independence between $g_1$ and $g_2\circ f_\theta$ is needed for identification. This is achieved if the model is overparameterized, and, in addition, we can ensure that $g_2\circ f_\theta$ remains sufficiently complex. 
As a concrete instantiation where this is the case, we could have $g_1,g_2\in\mathcal G$ with $\cG$ being the class of degree $k$ polynomials, and $f_\theta$ being of degree $k'>k$. 
More generally, in practical settings with overparameterized models, we expect incongruence to persist beyond the linear setting. In particular, there is no reason to believe that there is any structural similarity in the structural relationship between features and label, and the nature of performative effects. Thus, it is reasonable to assume that $g_2\circ f_\theta$ inherits the complexity of $f_\theta$.

\subsection{Identifiability from classification}

A second ubiquitous source of incongruence that we can exploit for identification is the \emph{discrete} nature of predictions $\hat Y$ in the context of classification. The resulting discontinuity in the relationship between $X$ and $\hat Y$ enables us to disentangle the direct causal link between $X$ and $Y$ from the indirect link mediated by the prediction $\hat Y$.
This identification strategy is akin to the popular regression discontinuity design~\citep{lee10rdd} and relies on the assumption that all other variables in $X$  are continuously related to $Y$ around the discontinuities in $\hat Y$.
Together with the separability of the structural causal model, we can establish the following global identifiability result:

\begin{proposition}[Identifiability for discrete classification]\label{prop:discrete}
Assume that the effect of $X$ and $\hat{Y}$ on $Y$ are separable $g(X,\hat Y)= g_1(X) + g_2(\hat{Y}), \forall X, \hat{Y}$ for some differentiable functions $g_1$ and $g_2$. Further, suppose $X$ is a  continuous random variable and $\hat{Y}$ is a discrete random variable that takes on at least two distinct values with non-zero probability. Then, $\cM_Y$ is identifiable from observational data. 
\end{proposition}

Similar to Proposition~\ref{prop:over}, the separability assumption together with incongruence provides a way to separate the direct effect from the indirect effect of $X$ on $Y$. Separability is necessary in order to achieve global identification guarantees without randomness because the identification of entangled components without overlap is fundamentally hard. Thus, under violations of the separability assumptions, a regression discontinuity design only enables approximate identification of the causal effect locally around the discontinuity by comparing similar units right above and right below the threshold that obtained a different prediction. This means that reliable extrapolation away from the threshold is not possible without further assumptions. 

In general, the further from $f_\theta$ we aim to extrapolate, the more we rely on assumptions, and the more brittle our causal conclusions become to violations of that said assumptions. Akin to Section~\ref{sec:random} (if one can not be confident about the overlap being satisfied on all of $\hat Y$ for every $X$), we recommend being cautious when relying on supervised learning approaches to reason about the impact of substantial updates to the predictive model, even if we put aside concerns about data scarcity. Rather, we would recommend considering data-driven predictions as a tool to inform local updates to the predictive model in the context of gradual exploration so as to stay within a suitably chosen trust region around $f_\theta$.

\section{Empirical evaluation}
The three settings studied in the previous section described several natural scenarios where we can hope to answer the causal question outlined in Section~\ref{sec:setup} with a model learned using supervised learning. 
In this section, we investigate empirically how well the supervised learning solution $h_\text{SL}$ in~\eqref{eq:obj} is able to identify a transferable functional relationship with finite data. 

\paragraph{Methodology.}
We generated semi-synthetic data for our experiments, using a Census income prediction dataset from \url{folktables.org} \citep{ding2021retiring}.\footnote{Appendix~\ref{appendix:exp} contains additional experiments on other Census datasets and the Kaggle credit scoring dataset \citep{kaggle}, along with more experimental details.}  Using this dataset as a starting point, we simulate a training dataset and test dataset with distribution shift as follows:
First, we choose two different predictors $f_\theta$ and $f_\phi$ to predict a target variable of interest (e.g. income) from covariates $X$ (e.g. age, occupation, education, etc.). If not specified otherwise, $f_\theta$ is fit to the original dataset to minimize squared error, while $f_\phi$ is trained on randomly shuffled labels. Next, we posit a function $g$ for simulating the performative effects. Then, we generate a \emph{training} dataset of $(X, \hat{Y}, Y)$ tuples following the structural causal model in Figure \ref{fig:haty}, using the covariates $X$ from the original data, $g,$ and $f_\theta$ to generate $\hat Y$ and $Y$.  Similarly, we generate a \emph{test} dataset of $(X, \hat{Y}, Y)$ tuples, using $X, g, f_\phi$.
We assess how well supervised methods learn transferable functional relationships by fitting a model $h_\text{SL}$ to the training dataset and then evaluating the root mean squared error (RMSE) for regression and the accuracy for classification on the test dataset. In our evaluations we compare predicting from predictions ($\hat Y$ included as a feature) with performative-agnostic learning ($\hat Y$ not included as a feature). We visualize the standard error from 10 replicates with different random seeds and we include an in-distribution baseline trained and evaluated on samples of $\cD(f_\phi)$.

\begin{figure}[t]
    \captionsetup{font=footnotesize}
    \centering
    \begin{subfigure}[b]{0.242\textwidth}
            \centering
            \includegraphics[width=\textwidth]{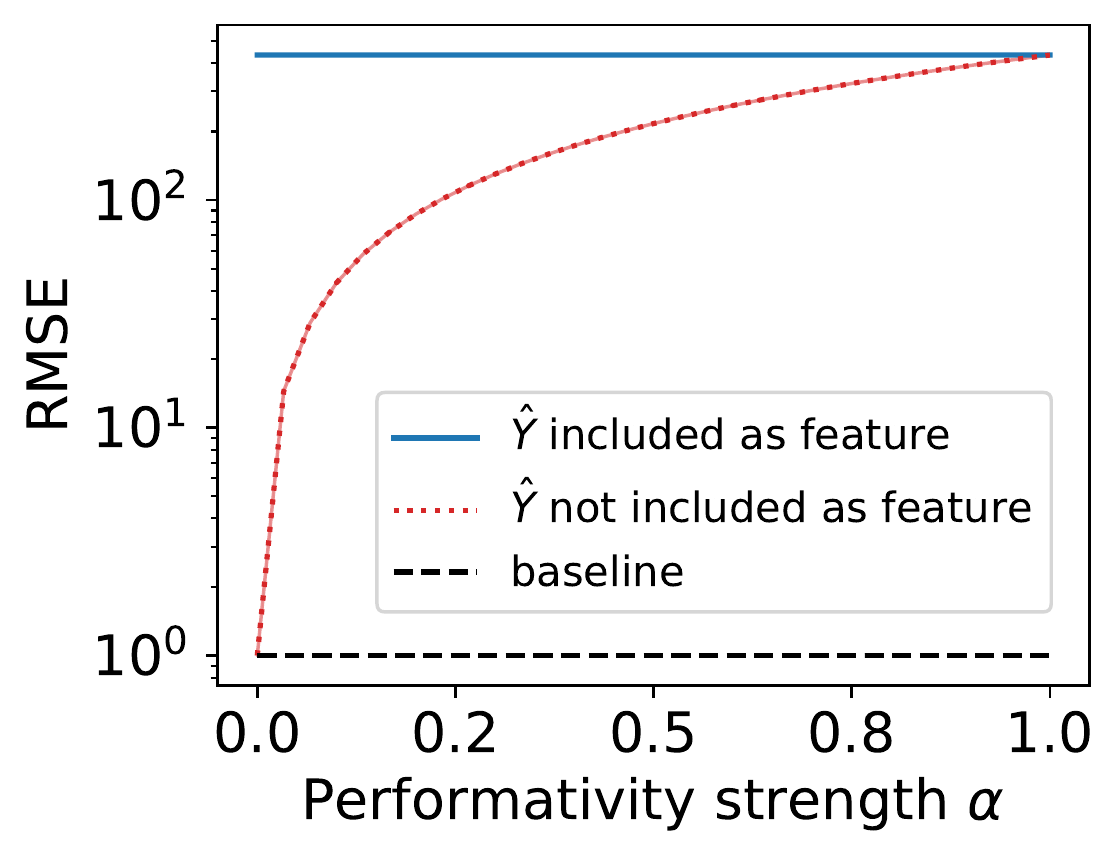}
            \caption{Non-identifiable setting}
        \end{subfigure}
        \hfill
        \begin{subfigure}[b]{0.242\textwidth}
            \centering
            \includegraphics[width=\textwidth]{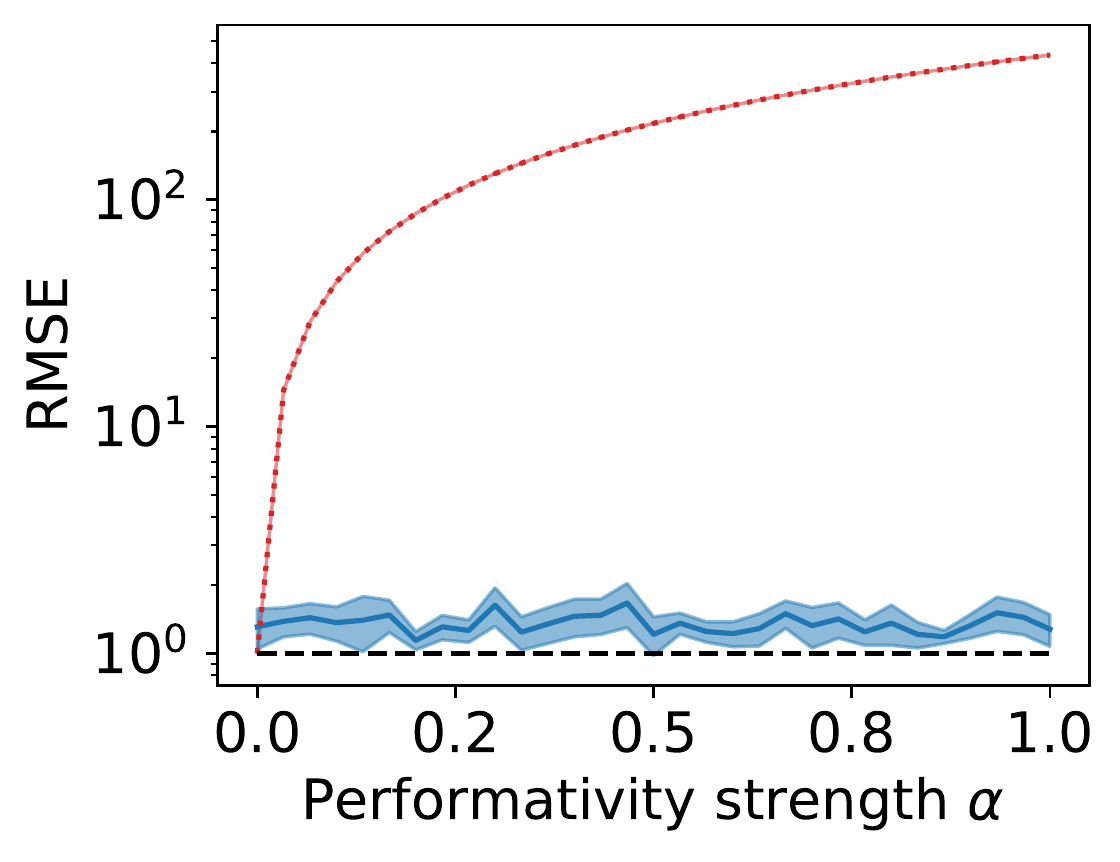}
            \caption{randomized $f_\theta$}
        \end{subfigure}
        \hfill
        \begin{subfigure}[b]{0.242\textwidth}
            \centering
            \includegraphics[width=\textwidth]{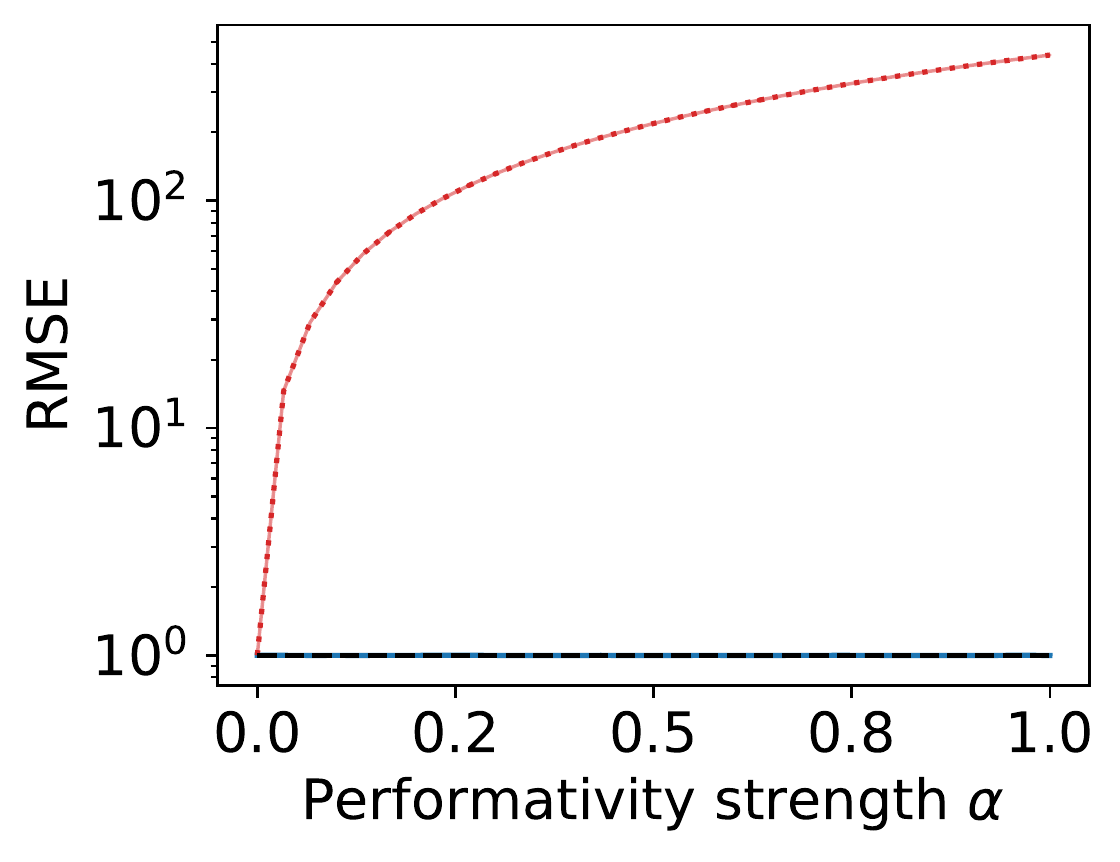}
            \caption{overparameterized $f_\theta$}
        \end{subfigure}
        \hfill\begin{subfigure}[b]{0.242\textwidth}
            \centering
            \includegraphics[width=\textwidth]{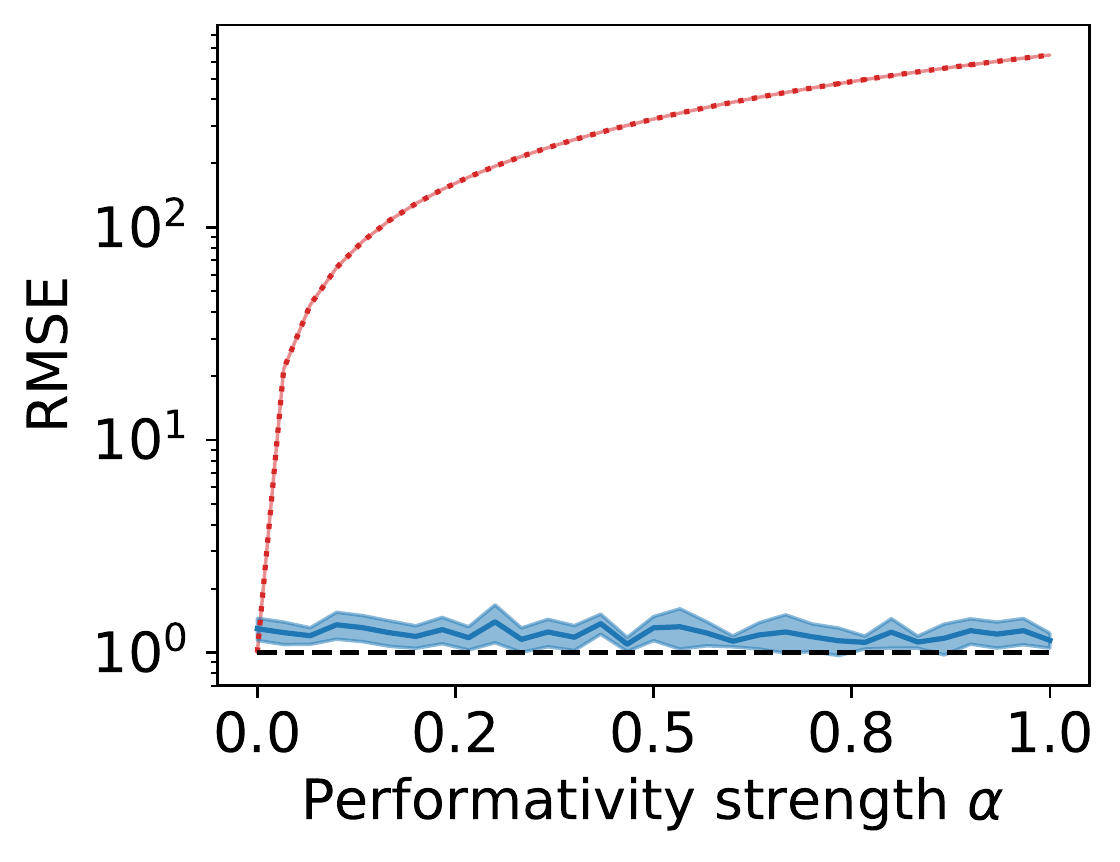}
            \caption{discrete $f_\theta$}
        \end{subfigure}
    \caption{\textbf{Extrapolation error of supervised learning with and without access to $\hat Y$.} (a) In the non-identifiable setting, adding $\hat Y$ as a feature harms generalization performance. (b)-(d) Randomization, overparameterization, and discrete predictions are each sufficient for avoiding this failure mode. Supervised learning obtains models robust to distribution shift when $\hat Y$ is given as a feature, while the extrapolation loss of the performativity-agnostic model grows with the strength of performativity.} 
    \label{fig:large_data}
\end{figure}

\subsection{Necessity of identification guarantees for supervised learning}
\label{sec:exp-random}
We start by illustrating why our identification guarantees are crucial for supervised learning under performativity.
Therefore, we instantiate the structural causal model in Figure~\ref{fig:haty} as
\begin{equation}g(X,\hat Y)=\beta^\top X + \alpha\hat Y
\label{eq:exp-g}
\end{equation}
with $\xi_Y\sim\cN(0,1)$. The coefficients $\beta$ are determined by linear regression on the original dataset. The hyperparameter $\alpha\geq 0$ quantifies the stength of  performativity that we vary in our experiments. The predictions $\hat Y$ are generated from a linear model $f_\theta$ that we modify to illustrate the resulting impact on identifiability. We optimize $h_\text{SL}$ in \eqref{eq:obj} over $\cH$ being the class of linear functions and  assume there are plenty of training data points ($N=200,000$) available.

We start by illustrating a failure mode of supervised learning in a non-identifiability setting (Proposition~\ref{propo:deterministic}).
Therefore, we let $f_\theta$ be a deterministic linear model fit to the base dataset ($f_\theta(X)\approx \beta^\top X$). This results in $\cM_Y$ not being identifiable from $\cD(f_\theta)$.  
In Figure~\ref{fig:large_data}(a) we can see that supervised learning indeed struggles to identify a transferable functional relationship from the training data. What we observe in the experiment is that the meta model returns $h_\text{SL}(X,\hat Y)=(1+\alpha)\hat Y$, instead of identifying $\cM_Y$ correctly as $g(X,\hat Y)$. Thus, this relationship is not preserved for our test model $f_\phi$, which leads to a high extrapolation error independent of the strength of performativity. While we only show the error for one $f_\phi$ in Figure~\ref{fig:large_data}(a), the error grows with the distance $d^2_{\cD_x}(f_\theta,f_\phi)$ between the training domain $\cD(f_\theta)$ and the target domain $\cD(f_\phi)$. In contrast, when the feature $\hat{Y}$ is not included, the supervised learning strategy returns $h_\text{SL}(X)=(1+\alpha)\beta^\top X$. The extrapolation loss of this performativity-agnostic model scales with the strength of performativity (c.f. Proposition~\ref{propo:ext}) and is thus strictly smaller than the error of the model that predicts from predictions in this example.

Once we leave the non-identifiable setting and move into the regime of our identification results (Proposition~\ref{prop:dp}-\ref{prop:discrete}), the benefit of including $\hat Y$ as a feature becomes apparent. To illustrate this, we reuse the same setup but modify the way the predictions in the training data are generated. In Figure~\ref{fig:large_data}(b) we use additive Gaussian noise to determine the predictions as $\hat Y=f_\theta(X)+\eta$ with $\eta\sim\cN(0,1)$. In Figure~\ref{fig:large_data}(c) we augment the input to $f_\theta$ with second-degree polynomial features to achieve overparameterization. In Figure~\ref{fig:large_data}(d) we round the predictions of $f_\theta$ to obtain discrete values. In all three cases, including $\hat Y$ as a feature is beneficial and allows the model to match in-distribution accuracy baselines, closing the extrapolation gap that is inevitable for performativity-agnostic learning.

\begin{figure}[t!]
\captionsetup{font=footnotesize}
    \centering
    \begin{subfigure}[b]{0.34\textwidth}
            \centering 
                \includegraphics[width=\textwidth]{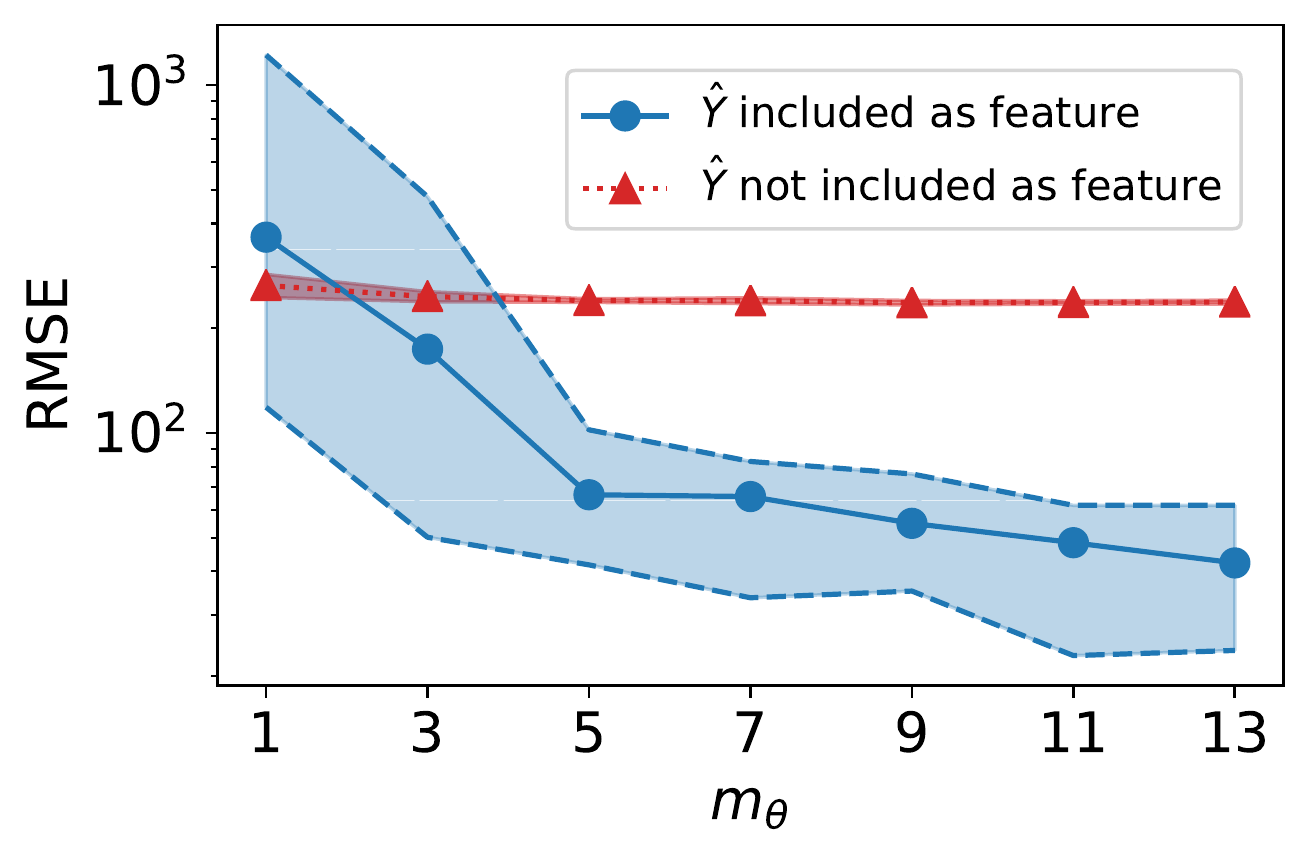}
            \caption{Varying functional complexity of $f_\theta$ }
        \end{subfigure}\hspace{1cm}
\begin{subfigure}[b]{0.34\textwidth}
            \centering
            \includegraphics[width=\textwidth]{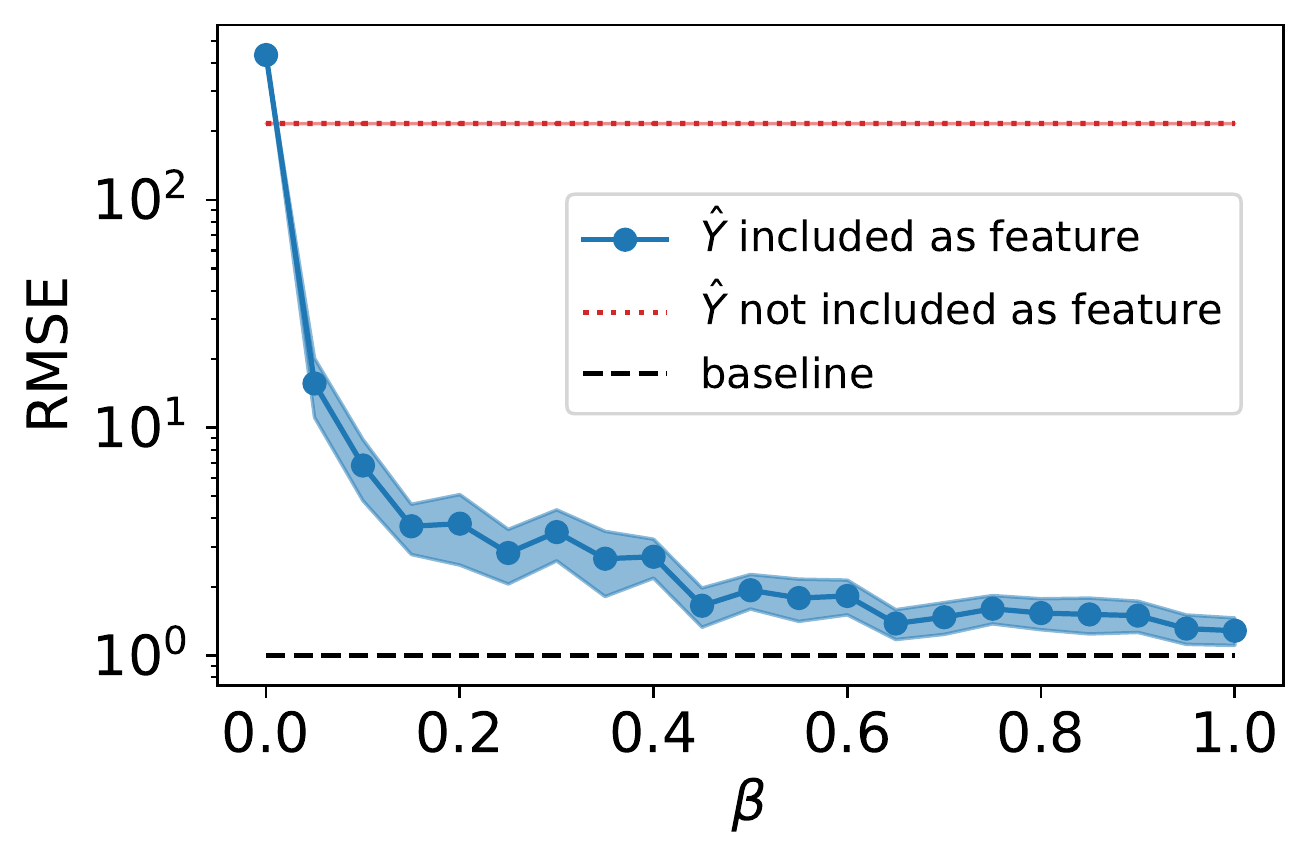}
            \caption{Varying degree of noise in $f_\theta$.}
        \end{subfigure}
        \caption{\textbf{Extrapolation performance with varying degrees of randomness and incongruence.} (a) We vary $m_\theta$ (the number of units in the hidden layer) of $f_\theta$.  Adding $\hat Y$ as a feature helps as soon as $f_\theta$ is overparameterized with respect to $g_1$. (b) We vary the magnitude of noise in the predictions of $f_\theta$. A small amount of noise is sufficient for identifiability.  Confidence sets show maximum and minimum across $10$ runs.
        }
        \label{fig:vary}
\end{figure}

\subsection{Strength of incongruence}

We next conduct an ablation study and investigate how the degree of overparameterization and the noise level for randomized $f_\theta$ impacts the extrapolation performance of supervised learning. Therefore, we consider the following instantiation of the structural equation model in Figure~\ref{fig:haty}:
\begin{equation}g(X,\hat Y)=g_1(X) + \alpha\hat Y
\label{eq:exp-gmain}
\end{equation}
with $\xi_Y\sim\cN(0,1)$. We fix the level of performativity at $\alpha=0.5$ for this experiment. We optimize $h_\text{SL}$ in \eqref{eq:obj} over $\cH$ (which we vary) and assume there are plenty of training data points ($N=200,000$) available.

\paragraph{Degree of overparameterization.} First, we explore the effect of overparameterization on the extrapolation error of $h_\text{SL}$. Therefore, we choose fully connected neural networks with a single hidden layer to represent the functions $g_1$, $f_\theta$, and $h_\text{SL}$. For the function $g_1$ and the hypothesis class $\cH$ we take a neural network with $m_g=3$ units in the hidden layer. We fit $g$ to the original dataset. Then, to simulate the degree of overparameterization of  $f_\theta$, we vary the number of neurons in the hidden layer of $f_\theta$, denoted $m_\theta$.
The resulting extrapolation performance of $h_\text{SL}$ on the test distribution is shown in Figure~\ref{fig:vary}(a). We can see how the extrapolation error of the learned model decreases with the complexity of $f_\theta$. In particular, as soon as $m_\theta>m_\phi$ there is a significant benefit to adding $\hat Y$ as a feature to the meta model. This corresponds to the regime where $\cM_Y$ becomes identifiable and $h_\text{SL}$ successfully recovers the transferable functional relationship in \eqref{eq:exp-gmain} as Proposition~\ref{prop:over} suggests. In turn, without adding $\hat Y$ as a feature the model suffers an inevitable extrapolation gap due to a concept shift that is independent of the properties of $f_\theta$.

\paragraph{Magnitude of noise.} In our second experiment on incongruence, we investigate the effect of the magnitude of additive noise added to the predictions in the linear model setting shown in Figure~\ref{fig:large_data}(b). Here $\cH$ and $g_1$ are linear functions and we vary the level of noise added to the predictions $f_\theta$. More specifically, we have  $\hat Y=f_\theta(X)+\beta \eta$ with $\eta\in\cN(0,1)$ where we vary $\beta$. The corresponding results can be found in Figure~\ref{fig:vary}(b). We see that even small amounts of noise are sufficient for identification and adding $\hat Y$ as a feature to our meta-machine learning model is effective as soon as the noise in $f_\theta$ is non-zero.

\subsection{Learning with finite data}

Recall that causal identification results are feasibility guarantees. They imply that $\cM_Y$ can be recovered from observational data in the limit of infinite data. However, in practical settings, we only get access to a finite set of data points from the training distribution $\cD(f_\theta)$. In the following, we show that supervised learning can successfully learn transferable functions $h_\text{SL}$ with only a few training data points, given that our identifiability conditions are satisfied. 

In Figure~\ref{fig:finite_data}(a)-(b) we consider the same setup as in Section~\ref{sec:exp-random}; we fix performativity strength at $\alpha=0.5$, and vary training set size. We find that only moderate dataset sizes are necessary for $h_\text{SL}$ to identify a model that is robust to performative distribution shifts.

In Figure~\ref{fig:finite_data}(c) we choose $N=5000$ and investigate the performance of supervised learning as we vary the distance between predictions from $f_\theta$ and $f_\phi$, i.e. the distribution shift between train and test set. We achieve this by interpolating the parameters of the predictive model in the test set between $f_\theta$ and $f_\phi'$ where the latter is trained on randomized labels as before.
We observe that the error and variance of $h_\text{SL}$ grow with the magnitude of the distribution shift. The reason is that failures in the meta model to identify the transferable causal model $\cM_Y$ become more pronounced as distribution shifts get larger. In addition, the variance in the extrapolation error grows with the distance from $f_\theta(x)$ due to data scarcity implied by the shape of the noise distribution in the randomized $f_\theta$.
This observation supports our recommendation to explore the parameter space gradually for policy optimization under performativity, instead of directly extrapolating to models $f_\phi$ that are substantially different from $f_\theta$.

\begin{figure}[t!]
    \captionsetup{font=footnotesize}
    \centering
            \begin{subfigure}[b]{0.32\textwidth}   
            \centering
            \includegraphics[width=\textwidth]{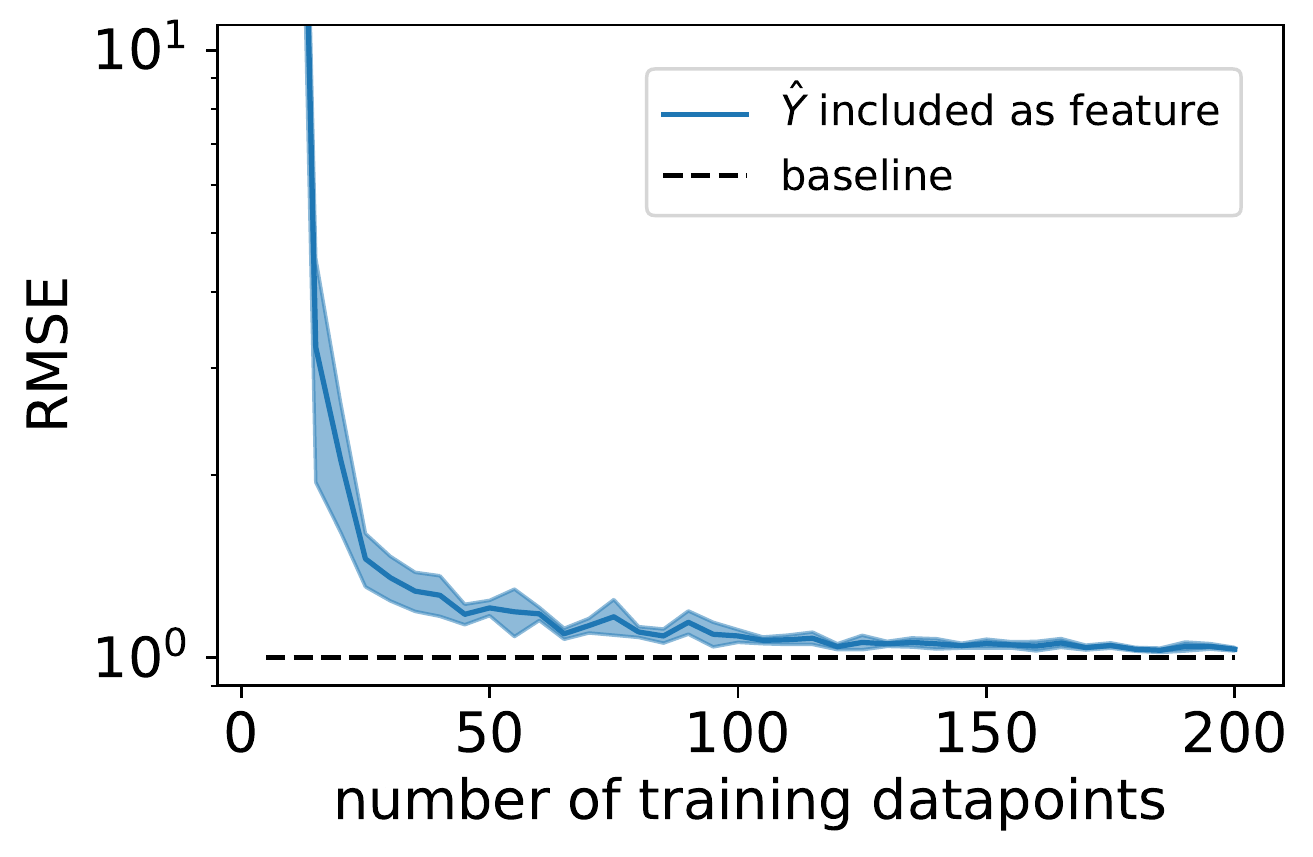}
            \caption{overparameterized $f_\theta$}
            \label{fig:finite_data_c}
        \end{subfigure}
        \hfill
        \begin{subfigure}[b]{0.32\textwidth}   
            \centering 
            \includegraphics[width=\textwidth]{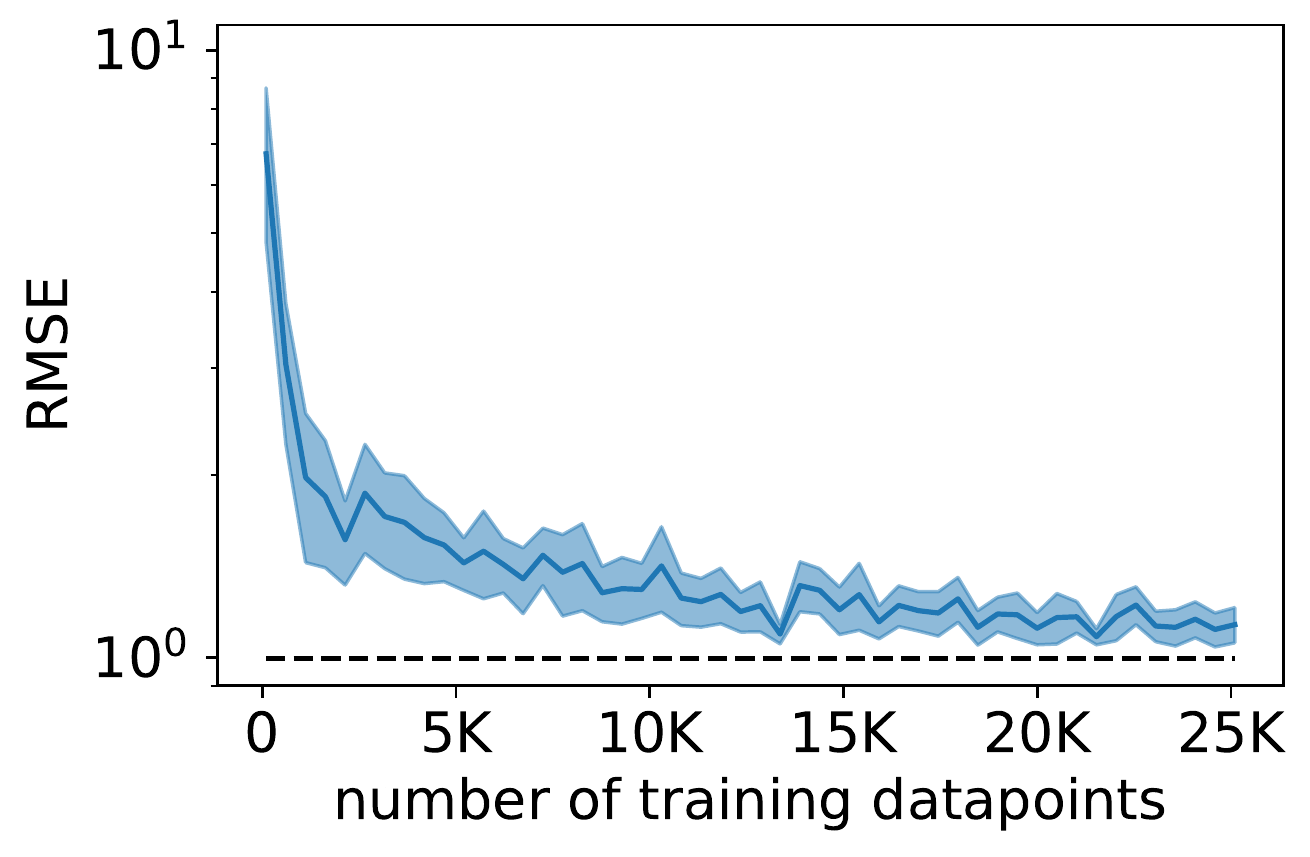}
            \caption{randomized $f_\theta$}
            \label{fig:finite_data_b}
        \end{subfigure}
        \hfill
        \begin{subfigure}[b]{0.317\textwidth}   
            \centering 
            \includegraphics[width=\textwidth]{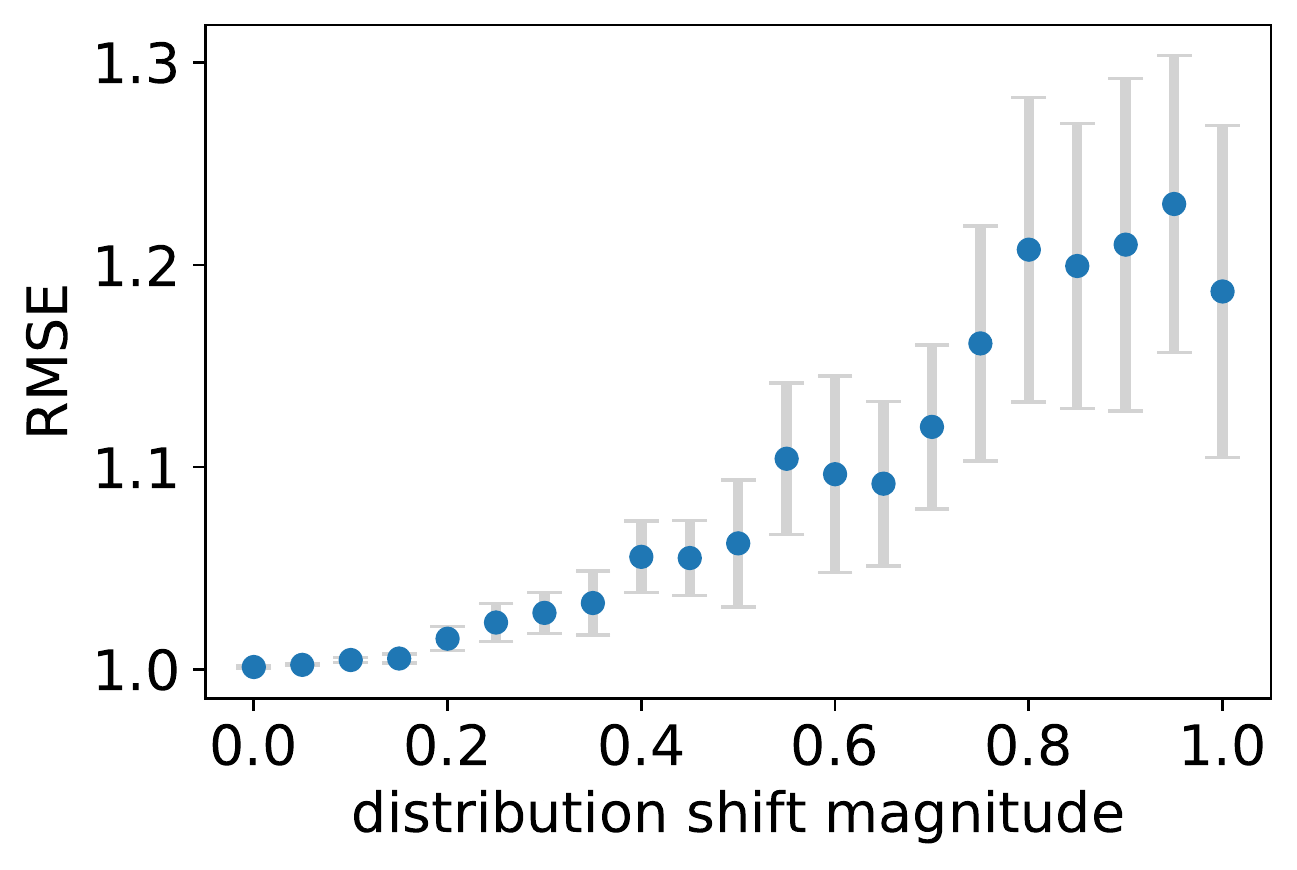}
            \caption{randomized $f_\theta$}
            \label{fig:finite_data_a}
        \end{subfigure}
    \caption{\textbf{Effect of training dataset size.} (a)-(b) With a moderate amount of training data, randomized decisions and overparameterization can find transferable functions $h_\text{SL}$.  (c) The variance in the extrapolation loss increases with the distribution shift magnitude. }
    \label{fig:finite_data}
\end{figure}


\section{Discussion}
\label{sec:interference}

This paper focused on identifying the causal effect of predictions on outcomes using observational data. We point out several natural situations where this causal question can be answered, but we also highlight situations where observational data is not sufficiently informative to reason about performative effects.
By establishing a connection between causal identifiability and the feasibility of anticipating performative effects using data-driven techniques, this paper contributes to a better understanding of the suitability of supervised learning techniques for explaining social effects arising from the deployment of predictive models in economically and socially relevant applications. 

\subsection{The message for data collection practices}

The positive results in this work demonstrate the value of logging information about the state of the deployed prediction function when collecting data for the purpose of supervised learning in social settings. Only if predictions are observed, they can be incorporated to anticipate the performative effects of future model deployments. 
In contrast, if the predictions are not available, $f_\theta$ disrupts the causal relationship between $X$ and $Y$ that we aim to understand, 
leading to unavoidable prediction errors. 
Thus, information about the deployed predictive model is crucial for an analyst hoping to understand the effects of deployed predictive models, for engineers hoping to foresee consequences of new model deployments, and for the research community studying performative phenomena. To date, such data is scarcely available in benchmark datasets, hindering the progress towards a better understanding of performative effects, essential for the reliable deployment of algorithmic systems in the social world.

\subsection{Limitations and extensions}

As we show in the experiments, the success of supervised learning approaches is closely tied to the corresponding identifiability conditions being satisfied. Identifiability can be possible if access to predictions is given. However, information about $\hat Y$ must not be understood as a green light to justify the use of supervised learning techniques to address performativity in full generality. The central assumption of our work is the causal model in Figure~\ref{fig:haty}. While it describes a rich and interesting class of performative prediction problems, it does not account for all mechanisms of performativity. This in turn gives rise to interesting questions for follow-up studies.

\paragraph{Covariate shift due to performativity.} Performative prediction~\citep{perdomo20pp} in full generality allows a predictive model $f_\theta$ to impact the joint distribution $P(X,Y)=P(Y|X)P(X)$ over covariates and labels. In this work, we have assumed that the distribution over covariates is unaffected by the attempt to predict $Y$ from $X$ and performative effects only surface in $P(Y|X)$. For our theoretical results, this implied that overlap in the $X$ variable across environments is trivially satisfied, which enabled us to pinpoint the challenges of learning performative effects due to the coupling between $X$ and $\hat Y$. For establishing identification under performative covariate shift additional steps are required to ensure identifiability.

\paragraph{Performative effects through social influence.}
A second neglected aspect are spill-over effects. Our causal model, proposed in Figure~\ref{fig:haty}, models performative effects at an individual level and relies on the stable unit treatment value assumption (SUTVA)~\citep{imbens2015causal}. There is no possibility for interference in the sense that the prediction of one individual can impact the outcome of his or her peers.
Such an individualistic perspective is not unique to our paper but prevalent in existing causal analyses and model-based approaches to performative prediction and strategic classification~\citep[e.g.,][]{Hardt:2016:SC,jagadeesan21a, miller20b,bechavod20help, ghalme21dark, harris21}. 
However, the presence of interference effects can have important implications for how causal effects should be estimated and interpreted~\citep[cf.][]{sober2006housing, tchetgen2012causal,Aronow17exposure,manski93}, which is yet unexplored in the context of performative prediction. In particular, in the presence of interference effects there is a crucial difference between unilateral interventions on the prediction of a single individual and interventions performed on the entire population, such as the deployment of a new model. This is akin to the important distinction between the individual causal effect and the overall causal effect in treatment effect estimation~\citep[e.g.,][]{tchetgen2012causal}. Concretely, for our model, interference implies that
\begin{equation}
\E[Y_i|X_i=x,\text{do}(f = f_\text{new}))]\;\neq\;\E[Y_i|X_i=x,\text{do}(\hat Y_i=f_\text{new}(x))]
\label{eq:ineq}
\end{equation}
and hence the consequences of intervening on $f$ on individual $i$ can no longer be explained solely by an intervention on the individual's prediction $\hat Y_i$. As a result,  approaches for microfounding performative effect based on models learned from simple, unilateral interventions\footnote{See \citet{bjoerk20} for a related field experiment.} result in different causal estimates than supervised learning based methods for identification as studied in this work.  While interference biases both estimates, a data-driven approach can implicitly pick up patterns of interference effects present in the data despite model-misspecifications, whereas individualistic models are blind to these effects. In Appendix~\ref{sec:interference} we provide an example where this is an advantage: our data-driven approach can exploit network homophily~\citep{goldsmith13} to explain the total causal effect of a model change on the outcome of an individual, whereas individualistic modeling misses out on the indirect component arising from neighbors influencing each other. This raises interesting questions for future work about how to best address interference in the context of performativity.

\paragraph{Performative effects beyond predictions.}
In our model we assumed that performative effects are mediated by the prediction. Thus, the potential outcome after an intervention on the predictive model $f$ could equally be explained by an intervention on the predictions $\hat Y=f(X)$. Under this assumption, treating $\hat Y$ as a feature allowed us to transform the original performative prediction problem with concept shift into a classical supervised learning problem with covariate shift. However, this general strategy is not limited to predictions $\hat Y$ as a sufficient statistic for the shift, but could as well be applied to other performativity-relevant properties of the prediction function $f_\theta$. These could be the relevance of individual model parameters for explaining strategic adaptation, any available information about counterfactual outcomes impacting individual behavior, or the exposure condition in the presence of spillover effects. 
Independent of how we decide to model performative effects, the validity of any causal claim will inevitably be limited to the scope of its assumptions. Extracting the relevant features to base the assumptions on requires domain knowledge---the more expert knowledge we can incorporate about how performative effects arise, the better we can pin down these statistics. This in turn simplifies the learning task and allows us to trade off assumptions with data requirements for causal identifiability.

\paragraph{Performativity in non-causal prediction.} Finally, our causal graph in Figure~\ref{fig:haty} posits that prediction is solely based on features $X$ that are causal for the outcome $Y$. This is a desirable situation in many practical applications because causal predictions disincentivize  gaming of strategic individuals manipulating their features~\citep{miller20b, bechavod20help} and offers explanations for the outcome that persist across environments~\citep{Rojas18transfer, buehlmann18}. Nevertheless, non-causal variables are often included as input features in practical machine learning prediction tasks. Establishing a better understanding for the implications of the resulting causal dependencies due to performativity could be an important direction for future work.

\subsection*{Acknowledgement}
The authors would like to thank Moritz Hardt and Lydia Liu for many helpful discussions throughout the development of this project, Tijana Zrnic, Krikamol Muandet, Jacob Steinhardt, Meena Jagadeesan and Juan Perdomo for detailed feedback on the manuscript, and Gary Cheng for helpful discussions about differential privacy. We are also grateful for a constructive discourse and valuable feedback provided by the reviewers that greatly helped improve the manuscript.

\newpage
\bibliographystyle{plainnat}
\bibliography{mybib}

\newpage

\appendix

\section{Social influence}
\label{sec:interference}

We have mentioned that the stable unit treatment value assumption (SUTVA)~\citep{imbens2015causal} underlying our causal analysis could be violated in certain performative prediction settings due to social influence and spill-over effects. We want to use this section to discuss the simple interference pattern illustrated in Figure~\ref{fig:interference} that generalizes our causal graph from Figure~\ref{fig:haty}. In particular, it allows for predictions of individual $j$ to impact the outcome of individual $i\neq j$:

\begin{equation}
\E[Y_i|X_i=x,\text{do}(f = f^*))]\;\neq\;\E[Y_i|X_i=x,\text{do}(\hat Y_i=f^*(x))]
\label{eq:ineq}
\end{equation}

\begin{figure}[t!]
 \captionsetup{font=small}
\centering
\begin{minipage}{0.38\textwidth}
\begin{tikzpicture}
    \node[state] (x1) at (0,0) {$X_i$};
    \node[state] (yhat1) [right =of x1, xshift=0.2cm] {$\hat Y_i$};
    \node[state] (y1) [right =of yhat1, xshift=0.2cm] {$Y_i$};
    \node[state] (x2) at (0,2) {$X_j$};
    \node[state] (yhat2) [right =of x2, xshift=0.2cm] {$\hat Y_j$};
    \node[state] (y2) [right =of yhat2, xshift=0.2cm] {$Y_j$};
    \node[state] (f) at (0.8,1) {$f_\theta$}; 
    \path (x1) edge (yhat1);
    \path (yhat1) edge (y1);
    \path (x2) edge (yhat2);
    \path (yhat2) edge (y2);
    \path (f) edge (yhat1);
    \path (f) edge (yhat2);
    \path (x1) edge[bend right=40] (y1);
    \path (x2) edge[bend left=40] (y2);
    \path[DarkGreen] (yhat1) edge (y2);
    \path[DarkGreen] (yhat2) edge (y1);
\end{tikzpicture}
\end{minipage}
\begin{minipage}{0.5\textwidth}
\begin{align*}
    X_i &= \xi_X &\xi_X\sim\cD_X\\
    \hat Y_i &=f_\theta(X_i)&\\
    Y_i&= g'(X_i,\hat Y_i, (\hat Y_j)_{j\in[n]}) + \xi_Y  \;\;&\xi_Y\sim\cD_Y
\end{align*}
\end{minipage}
\caption{Performative effects through prediction with interference (green arrows) for $n=2$.}
\label{fig:interference}
\end{figure}

Such effects could arise due to information flow about predictions in the population through social media platforms~\citep{chierichetti09} or verbal sharing. This in turn leads to indirect exposure that can bring forward phenomena of social comparison such as envy or encouragement~\citep{cikara11psycho}. 
In the presence of such interference effects the causal effect of intervening on the predictive model is no longer the same as the causal effect of intervening on an individual's prediction.
On the left-hand side of \eqref{eq:ineq} the predictions of all individuals are changed, whereas on the right-hand side only the prediction of individual $i$ changed. 

In the following, we want to highlight a setting where the data-driven strategy~\eqref{eq:obj} (that builds a model based on data collected under a population intervention) is able to implicitly pick up on these interference effects present in the data, whereas this information is not available from data collected under unilateral interventions.

%

\paragraph{Exposure modeling.}
To formally reason about interference effects through predictions, let's introduce $G_i$ as a sufficient statistic that mediates the dependency among units, such that $ Y_i\bigCI\{\hat Y_{j}\}_{j\neq i}| G_i$ for all $i\in[n]$. The statistic $G_i$ could encode the exposure of the entire population, the average prediction across the population, relevant predictions of the closest neighbors in a social network, or the relative value of $\hat Y_i$ compared to peers in a group. $G_i$ is typically constructed based on domain knowledge and is often assumed to be low-dimensional, limiting the complexity of interference among units and making the problem more tractable. What is unique to the prediction setting studied in this work, compared to randomized treatment assignments, is that predictions (and hence $G_i$) are typically correlated with the covariates and thus inherit structures present in the population, such as network homophily.

\paragraph{Homophily.}  \emph{Homophily} refers to the tendency for individuals to be similar to their neighbors which surfaces in our setting as correlations between the features of neighboring units~\citep{goldsmith13}. In the context of prediction, this further implies that a smooth prediction function $f_\theta$ will also exhibit correlations between predictions assigned to neighbors.  We formalize this through the following  property: \begin{equation}
\label{eq:homo}
|\E_{j\in N(i)}\hat Y_j- \hat Y_i|< \delta \text{ for every }i\in[n]\text{ and some small }\delta\geq0,
\end{equation}
where $N(i)$ denotes the set of neighbors of $i$. 
In the following, we want to highlight that in the presence of homophily the data-driven strategy~\eqref{eq:obj} (that builds a model based on data collected under a population intervention) is able to implicitly pick up on the interference effects present in the data, whereas this information is not available from data collected under unilateral interventions.
More specifically, assume interference effects are mediated by the average prediction in the neighborhood of an individual, i.e., $G_i=\E_{j\in N(i)}\hat Y_j$, then the outcome of individual $i$ can (at least partially) be explained by the prediction $\hat Y_i$ itself.
This results in a machine learning model \eqref{eq:obj} that will implicitly pick up the interference effects from the training data in order to explain the total causal effect of $f_\theta$ on the outcome. This is helpful for prediction, despite a misspecified causal graph. We illustrate this advantageous property over microfoundation models with the following example:

\paragraph{Linear-in-means model.} Consider the following linear-in-means model proposed by~\citet{manski93}:
\begin{equation}Y_i= g(X_i) + \alpha \hat Y_i + \beta G_i\quad \text{where}\quad G_i = \frac 1 {|N(i)|}{\textstyle\sum}_{j\in N(i)} \hat Y_j
\label{eq:llmm}
\end{equation} for some $\alpha>\beta>0$. This structural causal model  describes a setting of positive interference where spillover effects are mediated by the average prediction in the neighborhood of an individual and represent a dampened version of the direct effect. 
We can show that fitting a model $h$ to explain $Y$ as a function of $X$ and $\hat Y$ leads to smaller estimation error than learning $h$ from unilateral interventions.

\begin{proposition}
Given the structural causal model in \eqref{eq:llmm}. Assume the homophily assumption~\eqref{eq:homo} holds for $\delta=0$. Then, under the same identifiability conditions established in Section~\ref{sec:identify} for the SUTVA case. The supervised leanrning solution $h_\text{SL}$ will find a transferrable functional relationship even in the presence of interference.
\end{proposition}

Without explicitly measuring $G_i$, fitting a model to explain $Y$ as a function of $X$ and $\hat Y$ will result in $h(x,\hat y)= g(x) + (\alpha+\beta) \hat y$. This relationship transfers to the deployment of new models (assuming the underlying causal graph is fixed). In contrast, an estimate based on unilateral interventions would result in $h(x,\hat y)= g'(x) + \alpha \hat y$ which systematically underestimates the overall strength of performative effects and thus leads to a biased estimator.

\section{Proofs}

\begin{assumption}[positivity]
Consider the structural causal graph in Figure~\ref{fig:haty}. Positivity of $\hat Y$ over $\mathcal Y$  is satisfied if $P[\hat Y\in\mathcal S|X=x]>0$ for all $x\in\mathcal X$ and all sets $\mathcal S\subseteq \mathcal Y$ with positive measure, i.e., $P[\mathcal S]>0$.
\end{assumption}

\begin{lemma}
If the training distribution satisfies positivity of $\hat Y$ over $\mathcal Y'\subseteq\mathcal Y$, then $\E[Y|X=x,\hat Y=\hat y]$ is identifiable from the training data for any $\hat y\in\mathcal Y'$.
\label{lem:aux}
\end{lemma}

\subsection{Proof of Proposition~\ref{propo:ext}}

For notational convenience we write $h_\text{opt}(f_\theta)$ for $h_{f_\theta}^*$.
From realizability it follows that $\mathrm{R}_{f_\theta}(h_\text{opt}(f_\theta))=0$. Hence, the extrapolation loss is equal to 
\[\mathrm {Err}_{f_\theta\rightarrow f_\phi}(h_\text{opt}(f_\theta))=\mathrm{R}_{f_\phi}(h_\text{opt}(f_\theta))-\mathrm{R}_{f_\theta}(h_\text{opt}(f_\theta))=\mathrm{R}_{f_\phi}(h_\text{opt}(f_\theta))\]
and it remains to bound $\mathrm{R}_{f_\phi}(h_\text{opt}(f_\theta))$:
\begin{align}\mathrm{R}_{f_\phi}(h_\text{opt}(f_\theta))&
=\mathrm E_{x,y\sim \cD_\text{X,Y}(f_\phi)} \mathcal L(h_\text{opt}(f_\theta)(x),y)\\
&=\mathrm E_{x\sim \cD_X} \mathcal L(h_\text{opt}(f_\theta)(x),g(x,f_\phi(x)))\\
&=\mathrm E_{x\sim \cD_X} \mathcal L(g(x,f_\theta(x)),g(x,f_\phi(x)))
\end{align}
Further assuming that the loss function $\cL$ is $\mu$-strongly convex and $\gamma$-smooth in the second argument. Then,
\begin{align}
\mathrm{R}_{f_\phi}(h_\text{opt}(f_\theta))&\geq \frac \mu 2 \;\mathrm E_{x\sim \cD_X}  \left(g(x,f_\theta(x))-g(x,f_\phi(x))\right)^2 \\
\mathrm{R}_{f_\phi}(h_\text{opt}(f_\theta))&\leq \frac \gamma 2 \;\mathrm E_{x\sim \cD_X}  \left(g(x,f_\theta(x))-g(x,f_\phi(x))\right)^2 
\end{align} and the result follows.

\subsection{Proof of Lemma~\ref{lem:equal}}

Given that the risk minimization problem \eqref{eq:obj} is realizable and $\cM_Y$ is uniquely identifiable over $\cH$, the risk minimizer of the squared loss corresponds to $\E[Y|X=x,\hat Y=\hat y]$. Given the graph structure in Figure~\ref{fig:haty} there is no unobserved confounding and hence \[\E[Y|X=x,\hat Y=\hat y]=\E[Y|X=x,\text{do}(\hat Y=\hat y)]\implies h_\text{SL}(x,\hat y)=\cM_Y(x,\hat y).\]

\subsection{Proof of Proposition~\ref{propo:deterministic}}

Our goal is to show that $\cM_Y(X,\hat Y)$  can not uniquely be identified from $\cD(f_\theta)$ if $f_\theta$ is a deterministic function. The proof is by construction of a function $ h$ that fits the training data equally well, but does not generalize to data induced by a new prediction function.

Since $f_\theta$ is deterministic it holds that $\hat y =f_\theta(x)$ for all pasirs $x, \hat{y}$ in the observed data distribution. Thus, the function $h$ defined as follows
\[ h(x,\hat{y}) = \cM_Y(x, f_\theta(x))\] is equally compatible with the observational data. That is
\[\E_{\hat{Y}=f_\theta(X)}\E[Y|X, \hat{Y}] = \cM_Y(X, \hat{Y}) = {h}(X, \hat{Y}).\]  
Hence, $\cM_Y$ can not be distinguished from $ h$ based on observational data. It remains to show that $\cM_Y$ and $h$ do not coincide on new data. 

We assume that $Y$ non-trivially depends on $\hat Y$ and $\cY$ is not a singleton. This means, given some $x$, for every $\hat y$ there exists a $\hat y'\in\cY$ such that $g(x,\hat y)\neq g(x,\hat y')$. Define $f_\phi$ such that for and $x$ if $f_\theta(x)=\hat y$, we set $f_\phi(x) = \hat y'$. Then, 
\[\E_{\hat{Y}=f_\phi(X)}\E[Y|X, \hat{Y}] = \cM_Y(X, \hat{Y}) \neq {h}(X, \hat{Y})\]
which concludes the proof.

\subsection{Proof of Proposition~\ref{prop:dp}}

Output overlap guarantees that $P[\hat Y =f_\phi(x)| X=x]>0$ in the training distribution $\cD(f_\theta)$ for any $x\in\cX$. Identification and extrapolation to models $f_\phi$ that satisfy output overlap with $f_\theta$ follows from positivity (Lemma~\ref{lem:aux}) and the causal graph (\Cref{fig:haty}) which implies that there is no unobserved confounding and $\E[Y|X=x,\hat Y=\hat y] = \E[Y|X=x,\mathrm{do}(\hat Y=\hat y)]$.

\subsection{Proof of Proposition~\ref{prop:over}}

We first note that the overparameterization assumption implies that $g_1\in \cG$ and $f_\theta$ are linearly independent. Proof by contradiction: If not, then there exists $\alpha_1\ne\alpha'_1, \alpha_2\ne\alpha'_2$ and functions $g_1, g'_1\in \cG$ such that $\alpha_1 g_1(x) + \alpha_2 f_\theta(x) = \alpha'_1 g'_1(x) + \alpha'_2 f_\theta(x)$ for all $x$; it implies that $\alpha_1g_1(x)-\alpha'_1g'_1(x)   = (\alpha_2-\alpha'_2)  f_\theta(x) $ for all $x$. $\alpha_1g_1(x)-\alpha'_1g'_1(x) \in \cG$ since the class $\cG$ is closed under addition. This leads to a contradiction with the fact that $f_\theta(\cdot)$ is overparametrized with respect to $\cG$, which requires there exist no function $g\in\cG$ such that $g(x) = c f_\theta(x)$ for some $c>0$.

Next, since any $h\in \mathcal{H}$ is separable in $X$ and $\hat{Y}$, linear in $\hat{Y}$, and that $h(\cdot, \hat{y})\in \cG$ for any $\hat{y}$, we have that $h(x, \hat{y}) = g'_1(x) + \alpha' \hat{y}$ for some $g'_1\in \cG$ and some constant $\alpha'\in \mathbb{R}$. Therefore, finding $\cM_Y$  amounts to solve $g'_1, \alpha'$ from the observational data relationship $g_1(X) + \alpha \hat{Y} = g'_1(X) + \alpha' \hat{Y}$ subject to the constraint that $\hat{Y}=f_\theta(X)$. Plugging in the constraints gives $g_1(X) + \alpha f_\theta(X) = g'_1(X) + \alpha'f_\theta(X)$. This equation gives a unique solution that $g'_1=g_1$ and $\alpha'=\alpha$ if we have observation from all values of $X$, hence the identifiability of $\cM_Y$.



\subsection{Proof of Proposition~\ref{prop:discrete}}

The proof is inspired by~\citep{wang19bless} and \citep{puli2020causal}. Because $\hat Y$ is discrete and $\mathbb{E}[Y| \mathrm{do}(\hat{Y}=\hat{y}),X=x]$ is separable, we have that  $ \frac{\partial }{\partial x} \mathbb{E}[Y| \mathrm{do}(\hat{Y}=\hat{y}),X=x] = \frac{\partial }{\partial x} g_1(x) = \frac{\partial }{\partial x} \mathbb{E}[Y| \hat{Y}=\hat{y}',X=x]$ for any pair of $\hat{y}, x$ that is observable, i.e. $\hat{y}' = f_\theta(x)$. This implication is due to $g_2(\hat{y})$ being a piecewise constant function  (its partial derivative is zero with respect to $x$). Therefore, the function $\cM_Y$ is identifiable
\[\cM_Y(x,\hat{y}) = \mathbb{E}[Y|\mathrm{do}( \hat{Y}=\hat{y}),X=x] = \mathbb{E}[Y| \hat{Y}=\hat{y},X=x'] + \int_{x'}^{x} \frac{\partial }{\partial x} \mathbb{E}[Y| \hat{Y}=\hat{y}',X=x] \mathrm{d}x,\]
for any $\hat{y}' = f_\theta(x)$. This equation
establishes the identifiability of $\cM_Y$. It also implies that the solution of the risk minimization problem \eqref{eq:obj} must  coincides with $\cM_Y$ if $\cH$ satisfies the identifiability condition, i.e. $\cH$ contains only separable functions $g(x, \hat{y})$ with differentiable $g_1, g_2$; these constraints implies the uniqueness of solution to the risk minimization problem, hence the solution must coincide with $\cM_Y$.

\section{Experiment details and additional experiments} \label{appendix:exp}
\subsection{Data and Licenses}
Data in folktables was extracted from Census Bureau databases, which collected data in standardized surveys with consent.\footnote{documentation: \url{https://www.census.gov/programs-surveys/acs/microdata/documentation.html}.} The Census Bureau takes care to ensure that through their pre-processing of survey results, personally identifiable information is not included in their data releases.\footnote{Terms of service: \url{https://www.census.gov/data/developers/about/terms-of-service.html}.}

The income dataset with binary outcome variables used for the results in the main body of the paper is the ACSIncome task defined in folktables, with data from the 2018 Census from the state of California. The income dataset with continuous outcome variables is a modified version of ACSIncome that performs the same pre-processing, except it leaves the income target variable as a real number, rather than thresholding to produce a binary outcome. Additional experiments below (referred to as Census travel time) were conducted on the ACSTravelTime task defined in folktables, with data from the 2018 Census from the state of California. The features and preprocessing for these datasets can be found in the code documentation of \cite{ding2021retiring}.
The data contains 10 features and the target variable is a binary indicator for whether an individual became delinquent on a loan: 
\begin{description}
\item[X=][RevolvingUtilizationOfUnsecuredLines, age, NumberOfTime30-59DaysPastDueNotWorse, 	DebtRatio, MonthlyIncome, NumberOfOpenCreditLinesAndLoans, NumberOfTimes90DaysLate, NumberRealEstateLoansOrLines, NumberOfTime60-89DaysPastDueNotWorse, NumberOfDependents],
\item[Y=] SeriousDlqin2yrs
\end{description}

\subsection{Experimental details}
Machine learning models were trained using functionalities from sklearn \citep{scikit-learn} with default parameters if not specified otherwise. We use the class \texttt{LinearRegression} from \texttt{sklearn.linear$\_$model} for the linear models and the class \texttt{MLRRegressor} from \texttt{sklearn.neural$\_$network} for the fully connected neural network models. 

\begin{description}
\item[Training data:] The Census income dataset composes of 195665 datapoints, if not specified otherwise the full dataset was used for training.
\item[Test dataset:] We train $f_\phi$ on randomized labels. More precisely, we randomly shuffle the labels among data points in the original dataset to obtain $f_\phi$. This process leads a model that is different from $f_\theta$ which serves to test the extrapolation performance of our meta-machine learning model.  
\item[Overparameterization:] For the experiment in Figure~\ref{fig:large_data}(c) second degree polynomial features were included to achieve overparameterization (using \texttt{sklearn.preprocessing}), no further hyperparameters were set; all second-order terms were included in the overparameterization.
For the experiments in Figure~\ref{fig:vary}(a) we simulate the degree of overparameterization by working with neural networks and varying the number of neurons in the hidden layers using the parameter \texttt{
hidden$\_$layer$\_$sizes}. 
\item[Randomization:] For the randomized decision experiments we use Gaussian noise. If not specified otherwise it is drawn from $\cN(0,1)$. 
\item[Discretization:] To obtain discrete predictions, we round the prediction outputs of the linear model $f_\theta$ so we achieve $4$ distinct discrete values.
\item[Finite data:] The finite data experiments were conducted on datasets with a continuous target variable $Y$ (Census income), and performance was assessed using root mean squared error (RMSE). To simulate the effect of training set size we randomly subsample the original data to obtain a smaller training set size for our meta machine learning model. 
\item[Distirbution shift:]
We simulate different amounts of distribution shift by choosing $\phi' =\rho \theta + (1-\rho) \phi$
where $\theta$ are the parameters of the model trained on the original data, and $\phi$ are the parameters of a model trained on randomized labels. 
\item[Infrastructure:] Experiments were run on 4 CPU cores for a total of 200 hours.
\item[Baseline:] The baseline in the plots is the RMSE of a model trained on samples from the test set and evaluated on a validation set that is held out from the test set, but from the same distribution. Thus it represents a setting with no distribution shift.
\end{description}

\subsection{Additional experiments: Robustness to model misspecifications} 

We investigate the robustness of supervised learning to misspecification of $g$. Therefore we focus on discrete classification where $f_\theta$ and $f_\phi$ are binary predictors. We use gradient boosted decision trees implemented in sklearn \citep{scikit-learn} with the default hyperparameters.
Performance is assessed using classification accuracy. 
Experiments are performed for the dataset used in the main body, as well as the Census travel time and Kaggle credit score datasets~\citep{kaggle}.

Our theoretical result in Proposition~\ref{prop:discrete} depend on knowing the correct model class $\mathcal H$ to optimize $h_\text{SL}$ on. In this section we test the resiliency to model misspecification. Therefore, we define $\mathcal H$ to be the class of linear functions but construct a non-linear data generation process as follows
\begin{equation}
    g(x, \hat{y}) =  \hat{y} \text{ with probability } p, \text{and } g(x, \hat{y}) = g_1(x)  \text{ otherwise,}
\end{equation}
where $g_1(x)$ is a (possibly non-linear) function that maps $x$ to its original label $y$ in the original dataset, and $p \in [0, 1]$ is a hyperparameter for performativity strength that we vary. Like in the finite data experiments, we also vary the distance between predictions from $f_\theta$ and $f_\phi$, i.e. distribution shift magnitude.

\paragraph{Effect of distance between $f_\theta$ and $f_\phi$.} In Figure~\ref{fig:misspecification}(a) we investigate the effect of the distirbutionshift magnitude for $p=0.5$. The distribution shift magnitude is simulated by changing the data that $f_\phi$ is trained on. As in the finite data experiments, $f_\phi$ is fit on a noisy version of the original dataset, where we tune the level of noise and generate noisy labels $y'$ via
\[y' = y \text{ with probability } 1-\gamma, \text{ and } 1-y \text{ with probability } \gamma.\]
In other words, $\gamma$ parameterizes the distance between the predictors $f_\theta$ (fit to clean data) and $f_\phi$ fit to noisy data.  With $\gamma=0.5$ the label $y$ and $y'$ are uncorrelated.

We observe that despite misspecification, the meta model benefits of having access to $\hat y$ and the accuracy of $h_\text{SL}$ remains close to in-distribution accuracy as long as distribution shifts are not too large (specifically, until $f_\theta$ and $f_\phi$ become uncorrelated). 

\paragraph{Strength of performativity.} Next, we investigate the effect of varying $p$ for a fixed $\gamma=0.45$.
Figure \ref{fig:misspecification}(b) highlights that the benefit of adding $\hat y$ as a feature persists across almost all values of $p$. However, $h_\text{SL}$ is more prone to errors from model misspecification when performativity is very weak. This is intuitive, since 
$\hat{Y}$ is correlated with the outcome $Y$, and a misspecified $h_\text{SL}$ might be best off attributing this correlation to the causal link; in such extreme cases, the results suggests that accuracy is slightly improved by dropping $\hat{Y}$ as a feature.

\paragraph{Weaker accuracy of $f_\theta$.} Finally, we investigate the effect of varying $p$ for a model $f_\theta$ that is fit to noisy labels in Figure~\ref{fig:misspecification}(c). We see that if the accuracy of $f_\theta$ is reduced (by fitting to noisy labels), the superiority of performativity-agnostic learning for $p\rightarrow 0$ disappears.

In summary, we found qualitatively similar results across all datasets. Including $\hat{Y}$ as a feature outperforms not including it, even when little performativity is present.

\begin{figure}[b!]
\captionsetup{font=footnotesize}
    \centering
    \begin{subfigure}[b]{0.32\textwidth}
            \centering
            \includegraphics[width=\textwidth]{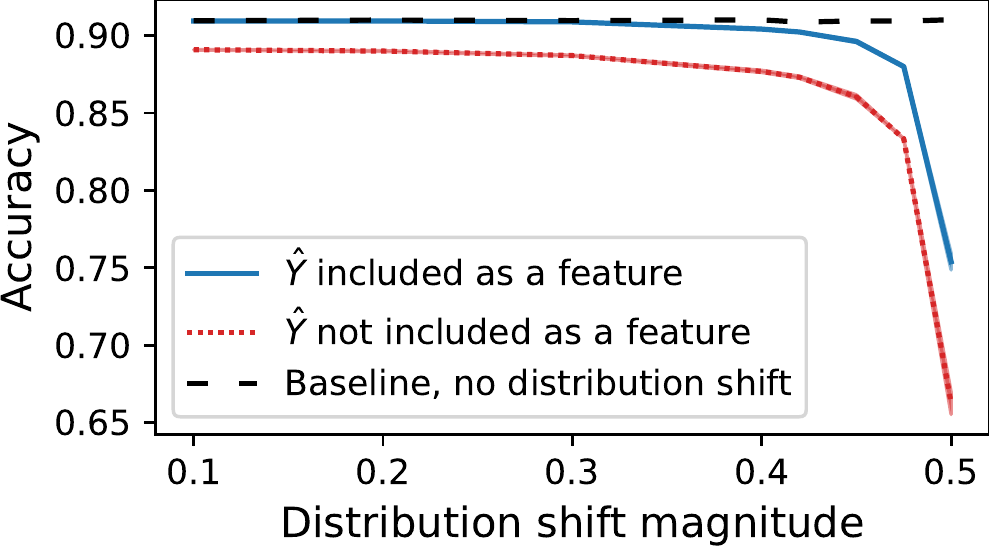}
        \end{subfigure}
        \hfill
        \begin{subfigure}[b]{0.32\textwidth}  
            \centering 
            \includegraphics[width=\textwidth]{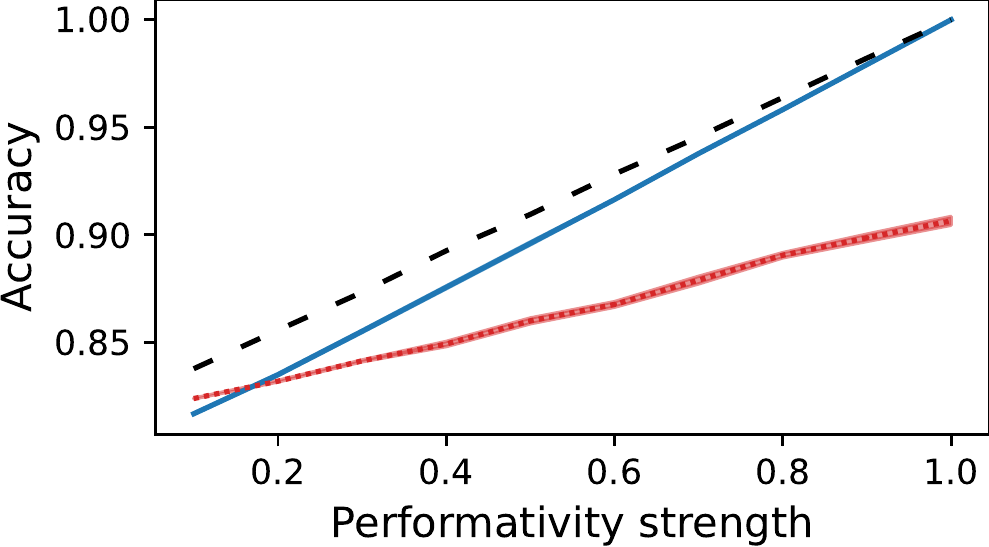}
        \end{subfigure}
        \hfill
        \begin{subfigure}[b]{0.32\textwidth}  
            \centering 
            \includegraphics[width=\textwidth]{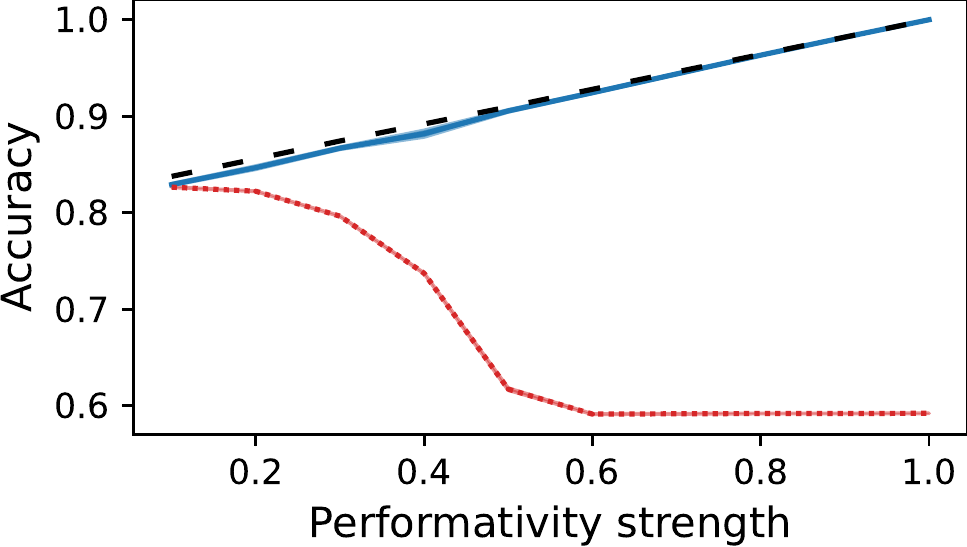}
        \end{subfigure}\\\vspace{0.5cm}
           \begin{subfigure}[b]{0.32\textwidth}
            \centering
            \includegraphics[width=\textwidth]{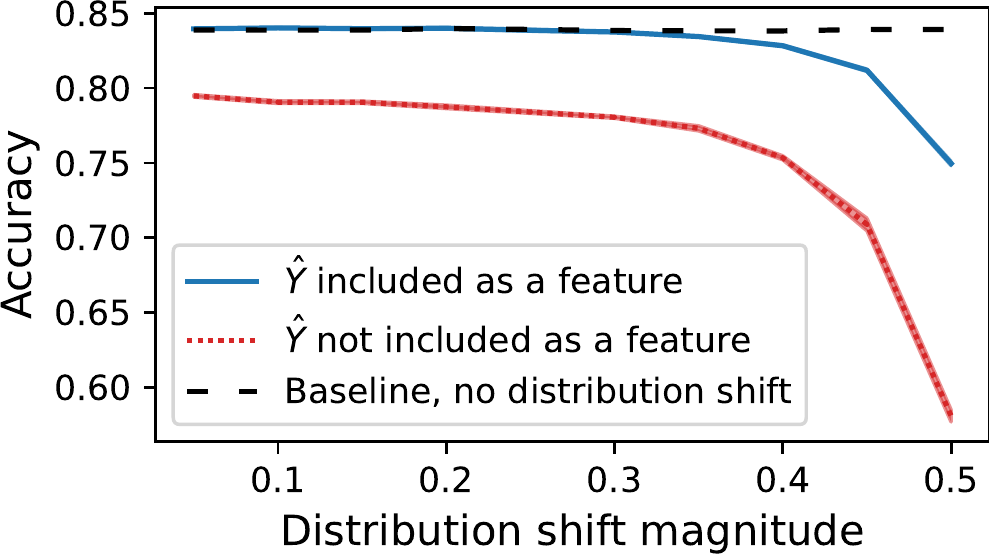}
        \end{subfigure}
        \hfill
        \begin{subfigure}[b]{0.32\textwidth}  
            \centering 
            \includegraphics[width=\textwidth]{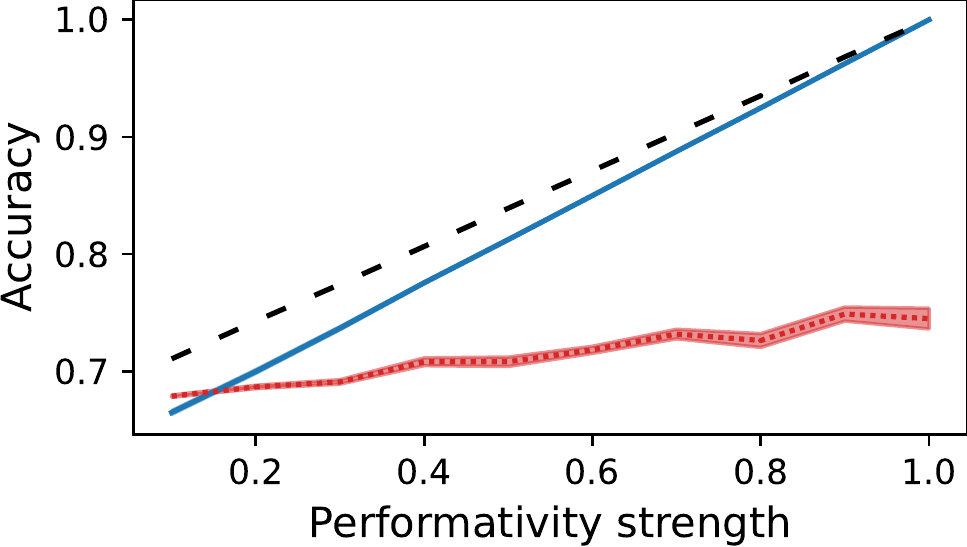}
        \end{subfigure}
        \hfill
        \begin{subfigure}[b]{0.32\textwidth}  
            \centering 
            \includegraphics[width=\textwidth]{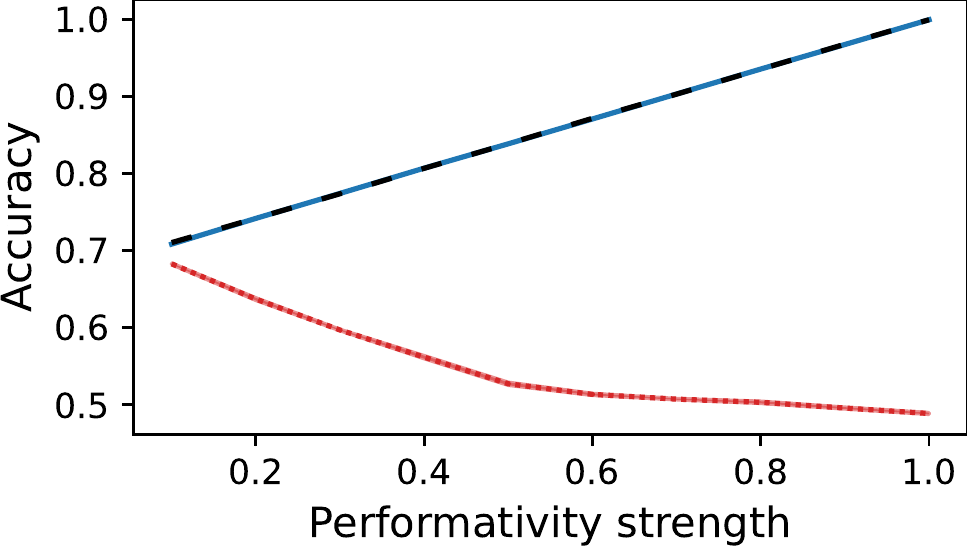}
        \end{subfigure}\\\vspace{0.5cm}
        
    \begin{subfigure}[b]{0.32\textwidth}
            \centering
            \includegraphics[width=\textwidth]{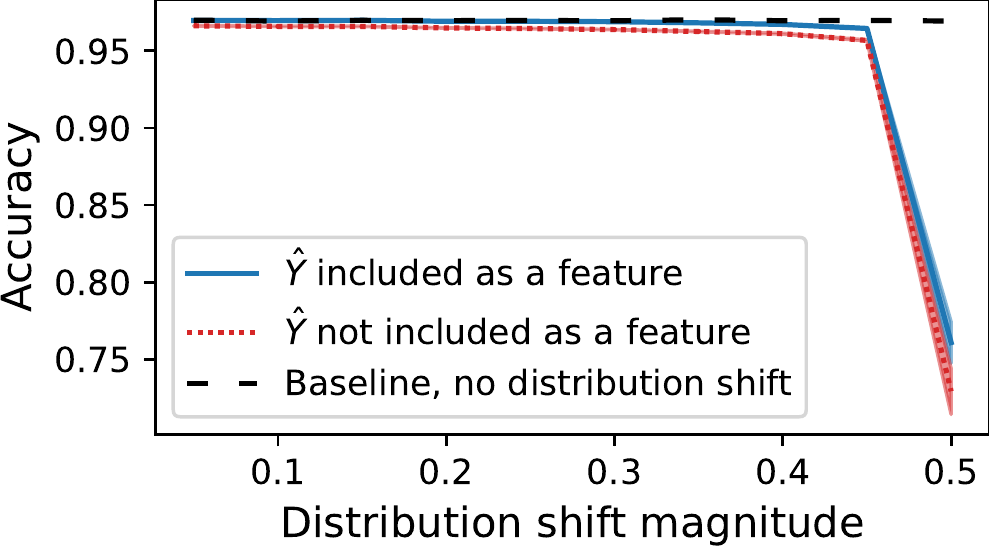}
            \caption{Distance between $f_\theta$ and $f_\phi$}
        \end{subfigure}
        \hfill
        \begin{subfigure}[b]{0.32\textwidth}  
            \centering 
            \includegraphics[width=\textwidth]{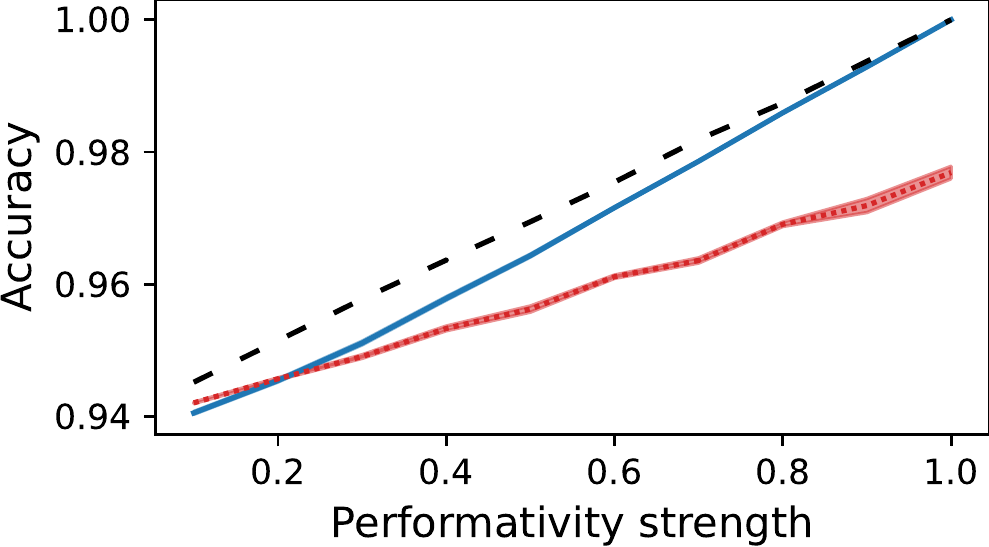}
            \caption{Standard setup: $f_\phi$ (test) fit to noisy labels}
        \end{subfigure}
        \hfill
        \begin{subfigure}[b]{0.32\textwidth}  
            \centering 
            \includegraphics[width=\textwidth]{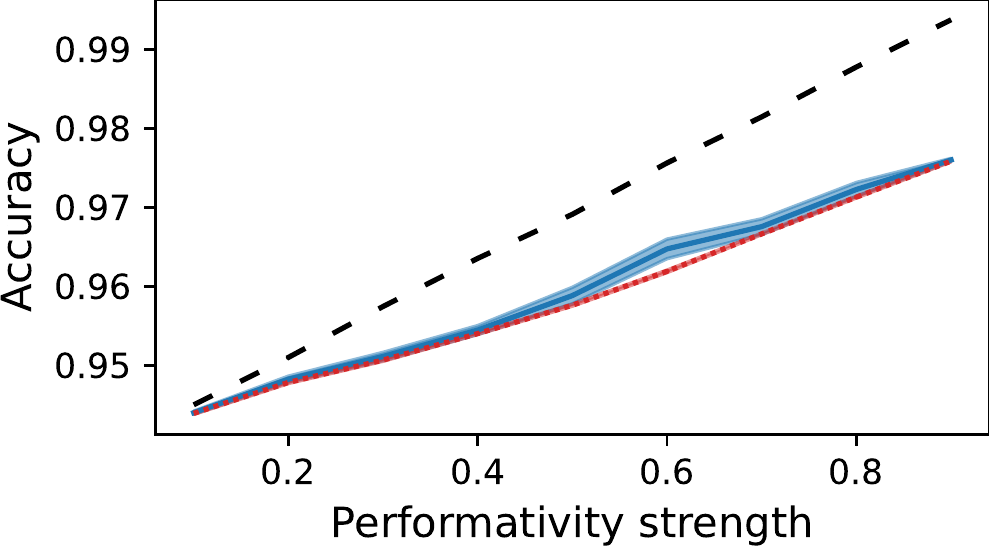}
            \caption{Reversed: $f_\theta$ (train) fit to noisy labels}
        \end{subfigure}
    \caption{\textbf{Performance degrades gracefully under model misspecification for different datasets.} Census income prediction dataset(first row), Census travel time dataset (second row), Kaggle credit score dataset (thrid row). (a) Accuracy on the test distribution (higher is better) is plotted against distribution shift magnitude; supervised learning remains accurate until the train and test predictors, $f_\theta$ and $f_\phi$, are uncorrelated (shift magnitude of 0.5). (b) Accuracy is plotted against performativity strength; despite model misspecification, accuracy is higher when $\hat{Y}$ is included as a feature, across most performativity strengths. (c) When the training set predictor $f_\theta$ is fit to random labels and is less accurate than $f_\phi$, including $\hat{Y}$ as a feature universally improves accuracy.} 
    \label{fig:misspecification}
\end{figure}

\newpage
\section{Societal impact}

The fact that predictions are performative and have an impact on the population they predict is a natural phenomenon observed in various applications. In this work, we discuss one dimension of performativity and investigate how to develop an improved causal understanding of these performative effects from data. Our intent is to develop this understanding from observational data in order to foresee potential negative consequences of a future model deployment before actually deploying it across an entire population. 
Typical machine learning approaches would not take these consequences into account when training a predictive model. At most, they would observe performative effects in a monitoring step after deploying a model and then decide post-hoc whether the model satisfies a given constraint. At this point, harm might already have been caused, even if unintentional. Naturally, though, any improvement in understanding can also be used with bad intent. Instead of being treated as a potential for harm to mitigate against, performative effects could also be instrumentalized by profit-maximizing firms or self-interested agencies in order to achieve their goals~\citep{hardt22power}. These goals might not always be aligned with social welfare and if the respective firm has high performative power, i.e. ability to influence performative effects, these actions hold the potential for social harm.

\end{document}